%% file: sample-sigconf.tex
\documentclass[sigconf]{acmart}

\input{math_commands.tex}

\usepackage{hyperref}
\usepackage{balance}
\usepackage{url}
\usepackage{mkolar_definitions}
\usepackage[ruled, vlined]{algorithm2e}
\usepackage{graphicx}
\usepackage{subcaption}
\usepackage{appendix}
\usepackage{amsthm,thmtools,thm-restate}
\usepackage{caption}
\usepackage{wrapfig}
\AtBeginDocument{%
  \providecommand\BibTeX{{%
    \normalfont B\kern-0.5em{\scshape i\kern-0.25em b}\kern-0.8em\TeX}}}

\theoremstyle{plain}

\newtheorem{claim}{Claim}
\newtheorem{definition}{Definition}

\newcommand{\ms}{MinSup} 

\DeclareCaptionLabelFormat{andtable}{#1~#2  \&  \tablename~\thetable}

\newcommand{\Pol}{{\Pi}}                 
\newcommand{\pol}{{\pi}}                 
\newcommand{\polnull}{{\pol_0}}          
\newcommand{\polres}{{\bar{\pol}}}       
\newcommand{\Polres}{{\bar{\Pol}}}       
\newcommand{\prop}{p}                    
\newcommand{\x}{x}                       
\newcommand{\Px}{P(\Xcal)}               
\newcommand{\y}{y}                       
\newcommand{\rew}{r}                     
\newcommand{\rewxy}{\delta}              
\newcommand{\rewhatxy}{\hat{\delta}}     
\newcommand{\rewregxy}{\hat{\delta}^{\theta}} 
\newcommand{\Prewxy}{P(\rew|\x,\y)}      
\newcommand{\Rew}{R}                     
\newcommand{\Rewips}{\hat{R}_{IPS}}      
\newcommand{\Rewipsaug}{\hat{R}_{IPS}^{\hat{\delta}}} 
\newcommand{\Rewdraug}{\hat{R}_{DR}^{\hat{\delta}}} 
\newcommand{\Rminsup}{\hat{R}_{MinSup}} 
\newcommand{\SD}[3]{\mathcal{D}_{#1}(#2|#3)} 
\newcommand{\SDx}[2]{\SD{\Xcal}{#1}{#2}}
\newcommand{\SDxpol}{\SDx{\pol}{\polnull}}
\newcommand{\Ucalpol}[1]{\Ucal(#1,\polnull)}
\newcommand{\Ucalpolc}[1]{\Ucal(#1,\polnull)^c}
\newcommand{\Ucalxpol}{\Ucalpol{\x}}

\copyrightyear{2020}
\acmYear{2020}
\setcopyright{acmlicensed}\acmConference[KDD '20]{Proceedings of the 26th ACM SIGKDD Conference on Knowledge Discovery and Data Mining}{August 23--27, 2020}{Virtual Event, CA, USA}
\acmBooktitle{Proceedings of the 26th ACM SIGKDD Conference on Knowledge Discovery and Data Mining (KDD '20), August 23--27, 2020, Virtual Event, CA, USA}
\acmPrice{15.00}
\acmDOI{10.1145/3394486.3403139}
\acmISBN{978-1-4503-7998-4/20/08}

\begin{document}
\fancyhead{}

\title{Off-policy Bandits with Deficient Support}

\author{Noveen Sachdeva}
\affiliation{%
  \institution{International Institute of Information Technology}
  \city{Hyderabad, India}
}
\email{ernoveen@gmail.com}
\authornote{Work done during internship at Cornell University.}
\authornotemark[2]

\author{Yi Su}
\affiliation{%
  \institution{Cornell University}
  \city{Ithaca, NY, USA}
}
\email{ys756@cornell.edu}
\authornote{Equal contribution.}

\author{Thorsten Joachims}
\affiliation{%
  \institution{Cornell University}
  \city{Ithaca, NY, USA}
}
\email{tj@cs.cornell.edu}

\begin{abstract}
Learning effective contextual-bandit policies from past actions of a deployed system is highly desirable in many settings (e.g. voice assistants, recommendation, search), since it enables the reuse of large amounts of log data. State-of-the-art methods for such off-policy learning, however, are based on inverse propensity score (IPS) weighting. A key theoretical requirement of IPS weighting is that the policy that logged the data has "full support", which typically translates into requiring non-zero probability for any action in any context. Unfortunately, many real-world systems produce support deficient data, especially when the action space is large, and we show how existing methods can fail catastrophically. To overcome this gap between theory and applications, we identify three approaches that provide various guarantees for IPS-based learning despite the inherent limitations of support-deficient data: restricting the action space, reward extrapolation, and restricting the policy space. We systematically analyze the statistical and computational properties of these three approaches, and we empirically evaluate their effectiveness. In addition to providing the first systematic analysis of support-deficiency in contextual-bandit learning, we conclude with recommendations that provide practical guidance.
\end{abstract}



\maketitle

\section{Introduction}
Many interactive systems (e.g., voice assistants, recommender systems) can be modeled as \textit{contextual bandit} problems \citep{langford2008epoch}. In particular, each user request provides a context (e.g., user profile, query) for which the system selects an action (e.g., recommended product) and receives a reward (e.g., purchase, click). Such contextual-bandit data is logged in large quantities as a by-product of normal system operation  \citep{li2011unbiased, li2015counterfactual, joachims2017unbiased}, making it an attractive and low-cost source of training data. 
With terabytes of log data readily available in many online systems, a range of algorithms has been proposed for batch learning from such logged contextual-bandit feedback \citep{strehl2010learning,dudik2011doubly, Swaminathan/Joachims/15c, thomas2016data,farajtabar2018more,su2019cab, london2019bayesian}. However, as we will argue below, these algorithms require an assumption about the log data that makes them unsuitable for many real-world applications.

This assumption is typically referred to as the positivity or support assumption, and it is required by the Empirical Risk Minimization (ERM) objective that these algorithms optimize. Specifically, unlike in online learning for contextual bandits \citep{williams1992simple, agarwal2014taming}, batch learning from bandit feedback (BLBF) operates in the off-policy setting. During off-policy learning, the algorithm has to address the counterfactual question of how much reward each policy in the policy space would have received, if it had been used instead of the logging policy. To this effect, virtually all state-of-the-art off-policy learning methods for contextual-bandit problems rely on counterfactual estimators \citep{bottou2013counterfactual,dudik2011doubly, Swaminathan/Joachims/15c,thomas2016data,farajtabar2018more,su2019cab} that employ inverse propensity score (IPS) weighting to get an unbiased ERM objective. Unlike regression-based direct-modeling (DM) approaches that are often hampered by bias from model misspecification, IPS allows a controllable bias-variance trade-off through clipping and other variance-regularization techniques \citep{strehl2010learning,Swaminathan/Joachims/15c,london2019bayesian}. 

Unfortunately, IPS and its guarantee of unbiasedness break down when the logging policy does not have full support -- meaning that some actions have zero probability of being selected under the logging policy. In this case IPS can be highly biased. Full support is an unreasonable assumption in many real-world systems, especially when the action space is large and many actions have poor rewards. For example, in a recommender system with a large catalog (e.g. movies, music), it may be that only a small percentage of the actions have support under the logging policy. We will show that existing learning algorithms can fail catastrophically on such support deficient data. 

In this paper, we explore how to deal with support deficient log data in off-policy contextual-bandit learning. Since support deficiency translates into blind spots where we do not have any knowledge about the rewards, accounting for these blind spots during training is crucial for robust learning. We characterize three approaches for dealing with support deficiency. The first approach is to restrict the action space to those actions that have support under the logging policy. Second, we explore imputation methods that extrapolate estimated rewards to those blind spots. And, third, we restrict the policy space to only those policies that have limited exposure to the blind spots. To make the latter approach computationally tractable, we define a new measure of Support Divergence between policies, show how it can be estimated efficiently without closed-form knowledge of the logging policy, and how it can be used as a constraint on the policy space. We analyze the statistical and computational properties of all three approaches and perform an extensive empirical evaluation. We find that restricting the policy space is particularly effective, since it is computationally efficient, empirically effective at learning good policies, and convenient to use in practice.

\section{Related Work}
\label{gen_inst}
The problem of learning optimal policies by re-using logged data is referred to as off-policy learning. There has been considerable interest in developing efficient off-policy learning in the contextual bandit setting. However, to the best of our knowledge, no previous work has comprehensively investigated the problem of support deficiency, which is an important and pervasive problem in many real-world systems. The prior work on off-policy learning can be classified into two fundamentally different approaches. The first -- called Direct Modeling (DM) -- is based on a reduction to supervised learning, where a regression estimate is trained to predict rewards \citep{beygelzimer2009offset}. To derive a policy, the action with the highest predicted reward is chosen. A drawback of this simple approach is the bias that results from misspecification of the regression model. For real-world data, due to non-linearity or partial observability of the environment, regression models are often substantially misspecified. Hence, the DM approach often does not perform well empirically. 

The second approach is based on policy learning via ERM with a counterfactual risk estimator \citep{dudik2011doubly, Swaminathan/Joachims/15c, Swaminathan/Joachims/15d, su2019cab}. Inverse propensity score (IPS) weighting is one of the most popular estimators to be used as empirical risk. However, policy learning algorithms based on IPS and related estimators \citep{strehl2010learning,Swaminathan/Joachims/15c,Swaminathan/Joachims/15d, thomas2016data,london2019bayesian} require the assumption that the logging policy has full support for every policy in the policy space.  One exception is the work of \cite{liu2019off}. They relax the assumption to the existence of an optimal policy such that the logging policy covers the support of this optimal policy. However, this is an untestable assumption that does not provide guarantees for real-world applications.

More generally, batch learning from bandit feedback can be viewed as off-policy learning in the reinforcement learning (RL) literature, which considers learning optimal policies in the sequential decision-making setting. Similar to the approaches in contextual bandit learning, RL methods cluster into two categories: (1) value function based and (2) importance sampling based. For importance weighting based methods, such as policy gradient \citep{williams1992simple}, the variance of the underlying gradient grows exponentially with the horizon, making it highly undesirable. It is worth noting that this is fundamentally different for one-step contextual bandits, where importance-sampling based estimators are the most competitive ones. For value function based approaches \citep{jiang2016doubly, kumar2019stabilizing}, the objective is based on some estimate of the value function, either through fitting an MDP model and evaluating the value function based on the estimated MDP \citep{jiang2016doubly} (a.k.a model-based in RL) or using bootstrapping to find the value function for the optimal policy directly, such as Q-learning \citep{watkins1992q} (a.k.a model-free in RL). However, for model-based methods, the bias problem could be severe if a wrong model class is chosen, and the bias problem is very difficult to diagnose in general. On the other hand, for model-free methods, Sutton and Barto \citep{sutton2018reinforcement} identify a deadly triad of function approximation, bootstrapping, and off-policy learning. It emphasizes that function approximation equipped with Q-learning can diverge in the off-policy learning setting, making most off-policy RL methods very conservative in extrapolation because of the severe error-propagation issue \citep{laroche2017safe, fujimoto2018off}. In this work, we focus on the contextual bandit problem, which has many direct applications in recommender systems and online search. We also believe that for developing better off-policy estimators in RL, it is fundamental to first understand how to handle these cases in the more tractable, contextual-bandit case.

In this paper, we explore three approaches to addressing off-policy learning with support deficiency and discuss how existing approaches could be fit into this framework. First, our conservative extrapolation method is related to the method proposed by \cite{liu2019off}. They focus on the correction of the state distribution by defining an augmented MDP, and pessimistic imputation is used to get an estimate for policy-gradient learning. Second, our method of restricting the policy space uses a surrogate for the support divergence of two policies that was previously used as control variate in the SNIPS estimator \citep{Swaminathan/Joachims/15d}. It also appeared in the Lagrangian formulation of the BanditNet objective \citep{Joachims/etal/18a} and in the gradient update of the REINFORCE algorithm \citep{williams1992simple}. This connection gives interesting new insight that the baselines used in policy-gradient algorithms not only help to reduce variance in gradients \citep{greensmith2004variance}, but that they also connect to the problem of support deficiency in the off-policy setting.

\section{Off-policy Learning with Deficient Support}
\label{headings}

We start by formally defining the problem of learning a contextual-bandit policy in the BLBF setting. Input to the policy are contexts $\x \in \Xcal$ drawn i.i.d.\ from a fixed but unknown distribution $\Px$. Given context $\x$, the system executes a possibly stochastic policy $\pol(\Ycal|\x)$ that selects an action $\y \in \Ycal$. For this context and action pair, the system observes a reward $\rew \in [\rew_{min},\rew_{max}]$ from $\Prewxy$. Given a space of policies $\Pol$, the reward of any policy $\pol \in \Pol$ is defined as
\begin{equation}
\Rew(\pi) = \EE_{\x\sim P(x)}\EE_{\y \sim \pol(\y|\x)} \EE_{\rew \sim \Prewxy}[\rew].	
\end{equation}
In the BLBF setting, the learning algorithm is given a dataset 
$$\Dcal:=\{\x_i, \y_i, \rew_i, \polnull(\y_i|\x_i)\}_{i=1}^n$$
of past system interactions which consists of context-action-reward-propensity tuples. The propensity $\polnull(\y_i|\x_i)$ is the probability of selecting action $\y_i$ for context $\x_i$ under the policy $\polnull$ that was used to log the data. We call $\polnull$ the logging policy, and we will discuss desired conditions on the stochasticity of $\polnull$ in the following. The goal of off-policy learning is to exploit the information in the logged data $\Dcal$ to find a policy $\hat{\pol} \in \Pol$ that has high reward $\Rew(\hat{\pol})$.

Analogous to the ERM principle in supervised learning, off-policy learning algorithms typically optimize a counterfactual estimate $\hat{\Rew}(\pol)$ of $\Rew(\pol)$ as the training objective \citep{li2011unbiased,li2015counterfactual,bottou2013counterfactual,Swaminathan/Joachims/15c}. 
\begin{equation}
    \hat{\pol} = \argmax_{\pol \in \Pol} [ \hat{\Rew}(\pol)]
\end{equation}
For conciseness, we ignore additional regularization terms in the objective \citep{Swaminathan/Joachims/15c}, since they are irrelevant to the main point of this paper. As for counterfactual estimator $\hat{\Rew}(\pol)$, most algorithms rely on some form of IPS weighting \citep{ strehl2010learning,dudik2011doubly,Swaminathan/Joachims/15c,Swaminathan/Joachims/15d,wang2016optimal,su2019cab} to correct the distribution mismatch between the logging policy $\polnull$ and each target policy $\pol \in \Pol$.
\begin{equation}
\label{est:ips}
	\Rewips(\pol) = \frac{1}{n}\sum_{i=1}^n\frac{\pol(\y_i|\x_i)}{\polnull(\y_i|\x_i)}\rew_i.
\end{equation}  
A crucial condition for the effectiveness of the IPS estimator (and similar estimators like SNIPS \cite{Swaminathan/Joachims/15d}) is that the logging policy $\polnull$ assigns non-zero probability to all actions that have non-zero probability under the target policy $\pol$ we aim to evaluate. This condition is known as positivity or full support, and it is defined as follows. 
\begin{definition}[Full support]
The logging policy $\polnull$ is said to have full support for $\pol$ when $\polnull(y|x)>0$ for all actions $y\in \Ycal$ and contexts $x\in \Xcal$ for which $\pol(y|x)>0$.
\end{definition}
It is known that the IPS estimator is unbiased, $\EE_{\Dcal}[\Rewips(\pol)] = \Rew(\pol)$, if the logging policy $\polnull$ has full support for $\pol$ \citep{li2011unbiased}.

To ensure unbiased ERM learning, algorithms that use the IPS estimator require that the logging policy $\polnull$ has full support for all policies $\pol \in \Pol$ in the policy space. For sufficiently rich policy spaces, like deep-networks $f_w(\x,\y)$ with softmax outputs of the form
\begin{equation}
\label{softmax_policy}
\pol_w(\y|\x) = \frac{exp(f_w(\x,\y))}{\sum_{\y' \in \Ycal}exp(f_w(\x,\y'))},
\end{equation}
this means that the logging policy $\polnull$ needs to assign non-zero probability to every action $\y$ in every context $\x$. This is a strong condition that is not feasible in many real-world systems, especially if the action space is large and many actions have poor reward.

If the support requirement is violated, ERM learning can fail catastrophically. We will show in the following that the underlying reason is bias, not excessive variance that could be remedied through clipping or variance regularization \citep{strehl2010learning,Swaminathan/Joachims/15c}. To quantify how support deficient a logging policy is, we denote the set of unsupported actions for context $\x$ under $\polnull$ as
$$\Ucalxpol:=\{\y\in\Ycal|\polnull(\y|\x)=0\}.$$

The bias of the IPS estimator is then characterized by the expected reward on the unsupported actions.
\begin{restatable}{propos}{ipsbias}
\label{ipsbias}
	Given contexts $\x \sim P(\Xcal)$ and logging policy $\polnull(\Ycal|\x)$, the bias of $\Rewips$ for target policy $\pol(\Ycal|\x)$ is equal to the expected reward on the unsupported action sets, i.e., $$bias(\Rewips(\pi))=\EE_x\bigg[-\sum_{y\in \Ucalxpol}\pol(\y|\x) \delta(x,y)\bigg].$$
\end{restatable}

The proof is provided in Appendix~\ref{Appendix: proofapp1}. From Proposition~\ref{ipsbias}, it is clear that support deficient log data can drastically mislead ERM learning. To quantify the effect of support deficiency on ERM, we define the support divergence between a logging policy $\polnull$ and a target policy $\pol$ as follows.

\begin{definition}[Support Divergence] For contexts $x\sim P(\Xcal)$ and any corresponding pair of target policy $\pol$ and logging policy $\polnull$, the Support Divergence is defined as
\begin{equation}
\SDxpol := \EE_{\x \sim P(\Xcal)}\bigg[\sum_{y\in \Ucalxpol} \pol(\y|\x)\bigg].
\end{equation} 	
\end{definition}
With this definition in hand, we can quantify the effect of support deficiency on ERM learning for a policy space $\Pol$ under logging policy $\polnull$.

\begin{restatable}{thm}{ipsermexp}
\label{ipsermexp}
For any given hypothesis space $\Pi$ with logging policy $\pi_0\in \Pi$, there exists a reward distribution $\Pcal_r$ with support in $[r_{min}, r_{max}]$ such that in the limit of infinite training data, ERM using IPS over the logged data $\Dcal\sim P(\Xcal)\times \pi_0(\cdot|\Xcal)\times \Pcal_r$ can select a policy $\hat{\pol} \in \argmax_{\pi\in\Pol}\EE_{\Dcal}[\hat{R}_{IPS}(\pi)]$ that is at least $$(\rew_{max}-\rew_{min}) \max_{\pol \in \Pol}\SDxpol$$ suboptimal.
\end{restatable}

\begin{proof}
For any given hypothesis space $\Pi$ and logging policy $\pi_0$, define a deterministic reward distribution $\Pcal_r$ supported in $[r_{min}, r_{max}]$ as following: for all context $x$, $r(x,y)=\delta(x,y)=r_{min}$ for $y\in \Ucal(x,\pi_0)^c$ and $r(x,y)=\delta(x,y)=r_{max}$ for $y\in \Ucal(x,\pi_0)$. Let $\tilde{\pi}\in\argmax_{\pol \in \Pol}\SDxpol$ and $\pi^{*}\in\argmax_{\pi\in\Pi}R(\pi)$, then we have the following lower bound for $R(\pi^{*})$:
\begin{equation}
\begin{split}
    R(\pi^{*})&\geq R(\tilde{\pi})\\
    &=\EE_x\bigg[\sum_{y\in \Ucal(x,\pi_0)}r_{max}+\sum_{y\in\Ucal(x,\pi_0)^c} r_{min}\bigg]\\
    &=r_{max}\max_{\pol \in \Pol}\SDxpol + r_{min}(1-\max_{\pol \in \Pol}\SDxpol)
\end{split}
\end{equation}
where the first inequality follows from the definition of $\pi^{*}$, the first and second equality is based on the specific reward distribution $\Pcal_r$ and the definition of $\tilde{\pi}$.

In the following we will show that for any $\hat{\pi}$ learned by the expectation of ERM (or in the limit of infinite amount data), i.e., $\hat{\pi}\in\argmax \EE_{\Dcal}[\hat{R}_{IPS}(\pi)]$, $\hat{\pi}$ have the same support as $\pi_0$. 

\begin{equation}
\begin{split}
 \EE_{\Dcal}[\hat{R}_{IPS}(\pi)] &= \EE_x\bigg[ \sum_{y\in \Ucal(x,\pi_0)^c}\pi(y|x)r_{min}\bigg]\\
 &= r_{min} \EE_x\bigg[ \sum_{y\in \Ucal(x,\pi_0)^c}\pi(y|x)\bigg]\leq r_{min}
 \end{split}
\end{equation}
for all $\pi\in\Pi$, then it is easy to see $\pi_0\in\Pi$ is one of the solution of ERM. Actually for any $ \hat{\pi}\in\argmax\EE_{\Dcal}[\hat{R}^{\Dcal}_{IPS}(\pi)]$, $$\EE_x\bigg[ \sum_{y\in \Ucal(x,\pi_0)^c}\pi(y|x)\bigg]=1$$ and it gives us that any solution of the ERM has exactly the same support as $\pi_0$, then we have $R(\hat{\pi})=r_{min}$ for $\hat{\pi}\in\argmax\EE_{\Dcal}[\hat{R}_{IPS}(\pi)]$.

Combining the lower bound for $R(\pi^{*})$ and $R(\hat{\pi})=r_{min}$, we have
\begin{equation}
\begin{split}
   R(\pi^{*}) - R(\hat{\pi}) &\geq r_{max}\max_{\pol \in \Pol}\SDxpol \\
   &+ r_{min}(1-\max_{\pol \in \Pol}\SDxpol)-r_{min}\\ &=(\rew_{max}-\rew_{min}) \max_{\pol \in \Pol}\SDxpol
\end{split}
\end{equation}
\end{proof}
To illustrate the theorem, consider a problem with rewards $\rew \in [-1,0]$. Furthermore, consider a policy space $\Pol$ that contains a good policy $\pol_g$ with $\Rew(\pol_g)=-0.1$ and a bad policy $\pol_b$ with $\Rew(\pol_b)=-0.7$. 
If policy $\pol_b$ has support divergence
$\SD{\Xcal}{\pol_b}{\polnull}=0.6$ or larger, 
then ERM may return the bad $\pol_b$ instead of $\pol_g$ even with infinite amounts of training data.

Note that it is sufficient to merely have one policy in $\Pol$ that has large support deficiency to achieve this suboptimality. It is therefore crucial to control the support divergence $\SDxpol$ uniformly over all $\pol \in \Pol$, or to account for the suboptimality it can induce. To this effect, we explore three approaches in the following. 

\subsection{Safe learning by restricting the action space}

The first and arguably most direct approach to reducing $\SDxpol$ is to disallow any action that has zero support under the logging policy. For the remaining action set, the logging policy has full support by definition. This restriction of the action set can be achieved by transforming each policy $\pol \in \Pol$ into a new policy that sets the probability of the unsupported actions to zero.
\begin{equation}
\label{est:action_res}
    \pol(\y|\x) \longrightarrow \polres(\y|\x):=\frac{\pol(\y|\x)\ind_{\{\y\notin\Ucalxpol\}}}{1-\sum_{y'\in\Ucalxpol}\pol(\y'|\x)}
\end{equation}
This results in a new policy space $\Polres$. All $\polres \in \Polres$ have support divergence of zero $\SDx{\polres}{\polnull} = 0$ and ERM via IPS is guaranteed to be unbiased.

While this transformation of the policy space from $\Pol$ to $\Polres$ is conceptually straightforward, it has two potential drawbacks. First, restricting the action space without any exceptions may overly constrain the policies in $\Polres$. In particular, if the optimal action $\y^*$ for a specific context $\x$ does not have support under the logging policy, no $\polres \in \Polres$ can ever choose $\y^*$ even if there are many observations of similar $\y$'s on similar context $\x'$. The second drawback is computational. For every context $\x$ during training and during testing, the system needs to evaluate the logging policy $\polnull(\y|\x)$ to compute the transformation from $\pol$ to $\polres$. This can be prohibitively expensive especially at test time, where -- after multiple rounds of off-policy learning with data from previously learned policies -- we would need to evaluate the whole sequence of previous logging policies to execute the learned policy.

\subsection{Safe learning through reward extrapolation}

As illustrated above, support deficiency is a problem of blind spots where we lack information about the rewards of some actions in some contexts. Instead of disallowing the unsupported actions like in the previous section, an alternative is to extrapolate the observed rewards to fill in the blind spots. To this effect, we propose the following augmented IPS estimator that imputes an extrapolated reward $\rewhatxy(\x,\y)$ for each unsupported action $\y \in \Ucalxpol$. 

\begin{equation}
\label{imp}
\Rewipsaug(\pol) = \frac{1}{n}\sum_{i=1}^n\bigg[\frac{\pol(\y_i|\x_i)}{\polnull(\y_i|\x_i)}\rew_i+\sum_{\y\in\Ucalpol{\x_i}}\pol(\y|\x_i)\rewhatxy(\x_i,\y)\bigg]
\end{equation}
In general, $\hat{\delta}(x,y)$ can be any function that maps $\Xcal\times\Ycal$ to $\RR$. The higher the quality of $\hat{\delta}(x,y)$, the better the evaluation accuracy of the associated augmented IPS estimator. In the following proposition, we formally characterize the bias of the augmented IPS estimator for any given reward extrapolation $\rewhatxy(\x,\y)$. We denote the mean of the reward $\rew$ for context $\x$ and action $\y$ with $\rewxy(x,y) = \EE_{\rew \sim \Prewxy}[\rew]$. Furthermore, let $\Delta(\x,\y):=\rewhatxy(\x,\y)-\rewxy(\x,\y)$ denote the error of the reward extrapolation for each $\x$ and $\y$.
\begin{restatable}{propos}{ipsregbias}
\label{ipsregbias}
	Given contexts $x_1, x_2, \dots, x_n$ drawn i.i.d from the unknown distribution $P(\Xcal)$, for action $y_i$ drawn independently from logging policy $\pi_0$ with probability $\pi_0(y_i|x_i)$, the bias of the empirical risk defined in Equation~\eqref{imp} is $\EE_x[\sum_{y\in\Ucalpol{x}}\pi(y|x)\Delta(x,y)]$. 
\end{restatable}

The proof is provided in Appendix~\ref{Appendix: proofapp2}. In this way we can learn in the original action and policy space, but mitigate the effect of the support deficiency by explicitly incorporating the extrapolated reward $\rewhatxy(\x,\y)$. We explore two choices for $\rewhatxy(\x,\y)$ in the following, which provide different types of guarantees.
\vspace{0.2cm}
\newline
\textbf{Conservative Extrapolation.} In practice, the logging policy is likely to put zero probability on actions that have low reward, since this minimizes the user impact of data collection. This means that precisely those bad actions are likely to not be supported in the logging policy. A key danger of blind spots regarding those actions is that naive IPS training will inadvertently learn a policy that selects those actions. This can be avoided by being maximally conservative about unsupported actions and imputing the lowest possible reward \cite{liu2019off}.
$$\forall x\: \forall \y \in \Ucalxpol: \rewhatxy(\x,\y)=\rew_{min}$$

Intuitively, by imposing the worst possible reward for the unsupported actions, the learning algorithm will aim to avoid these low-reward areas. However, unlike for the $\polres$ policies resulting from the restricted action space, the learned policy is not strictly prohibited from choosing unsupported actions -- it is merely made aware of the maximum loss that the action may incur. Note that for problems where $\rew_{min}=0$, the naive IPS estimator is identical to conservative extrapolation since the second term in Equation~\eqref{imp} is zero.
\vspace{0.2cm}
\newline
\textbf{Regression Extrapolation.}
Instead of extrapolating with the worst-case reward, we may have additional prior knowledge in the form of a model-based estimate that reduces the bias. In particular, we explore using a regression estimate $$\rewhatxy = \argmin_{\rewregxy}\frac{1}{n}\sum_{i=1}^n(\rewregxy(\x_i,\y_i)-\rew_i)^2$$ that extrapolates from the observed data $\Dcal$.
Typically, $\rewregxy$ comes from a parameterized class of regression functions (e.g. linear, deep networks).\footnote{In our experiments, we use deep neural networks with details shown in Appendix~\ref{Appendix:experiment}.} Other regression objectives could also be used, such as weighted linear regression that itself uses importance sampling as weights \citep{farajtabar2018more}. But, fundamentally, all regression approaches assume that the regression model is not misspecified and that it can thus extrapolate well. 

Note that the IPS part of Equation~\eqref{imp} can be exchanged for other estimators. In particular, we note that doubly robust (DR) \citep{dudik2011doubly} naturally performs a form of regression extrapolation. As the following decomposition shows, DR imputes the extrapolated reward $\rewhatxy(\x_i,\y)$ for the unsupported actions $y\in \Ucalpol{\x_i}$.
\begin{eqnarray}
\label{est:augdr}
    \Rewdraug \!\!&\!\!\!\!\!\!=\!\!\!\!\!\!&\!\! \frac{1}{n}\sum_{i=1}^n \bigg[\sum_{y\in\Ycal} \pol(\y|\x_i)\rewhatxy(\x_i,\y) + \frac{\pol(\y_i|\x_i)}{\polnull(\y_i|\x_i)}(\rew_i-\rewhatxy(\x_i,\y_i))\bigg]  \nonumber \\
    & \!\!\!\!\!\!=\!\!\!\!\!\! &\!\!\frac{1}{n}\sum_{i=1}^n\bigg[\sum_{\y\in\Ucalpolc{\x_i}}\!\!\!\!\!\!\!\!\pol(\y|\x_i)\rewhatxy(\x_i,\y) + \frac{\pol(\y_i|\x_i)}{\polnull(\y_i|\x_i)}(\rew_i-\rewhatxy(\x_i,\y_i)) \nonumber \\
    & &+\sum_{\y\in\Ucalpol{\x_i}}\!\!\!\!\!\!\!\!\pol(\y|\x_i)\rewhatxy(\x_i,\y)\bigg]
\end{eqnarray}
A similar decomposition also exists for the CAB \citep{su2019cab} estimator, showing that both DR and CAB belong to the class of Regression Extrapolation estimators. 

\begin{algorithm}[t]
\SetAlgoLined
input: original logged dataset $\Dcal$, replaycount $k$, reward estimate $\hat{\delta}(x,y)$; \\
output: additional augmented dataset $\Dcal'$\;
 initialization: $\Dcal' = \emptyset$ \;
 \For{$j=1,\dots,k$}{
  \For{$i=1,\dots,n$}{
  	Define $U_{x_i}$ to be the uniform distribution over $\Ucal(x_i,\pi_0)$\;
    Draw $y\sim U_{x_i}$\;
    $\Dcal' = \Dcal' \bigcup \{x_i, y, \hat{\delta}(x_i,y), \frac{1}{|\Ucal(x_i,\pi_0)|}\}$\;
  }{
  }
 }
\caption{Data Augmentation}
\label{data_aug}
\end{algorithm}
\noindent
\newline
\textbf{Efficient Approximation.}
Evaluating the augmented IPS estimator from Equation~\eqref{imp} can be computationally expensive if the number of unsupported actions in $\Ucalxpol$ is large. To overcome this problem, we propose to use sampling to estimate the expected reward on the unsupported actions, which can be thought of as augmenting the dataset $\Dcal$ with additional observations where the logging policy has zero support. In particular, we propose the data-augmentation procedure detailed in Algorithm~\ref{data_aug}.
With the additional bandit data $\Dcal'=\{\x'_j, \y'_j, \rewhatxy(\x'_j,\y'_j), \prop'_j\}_{j=1}^m$ from Algorithm~\ref{data_aug}, the new objective is \footnote{For the rest of the paper, we will use maximizing reward or minimizing loss interchangeably.}
\begin{equation}
\label{erm_samp}
	\argmin_{\pol \in \Pol} \left\{\frac{1}{n}\sum_{i=1}^n \frac{\pol(\y_i|\x_i)}{\polnull(\y_i|\x_i)}\rew_i +\frac{1}{m}\sum_{j=1}^m \frac{\pol(\y'_j|\x'_j)}{\prop'_j}\rewhatxy(\x'_j,\y'_j)\right\}
\end{equation}

In Appendix~\ref{Appendix:effapp}, we show that the empirical risk in Equation~\eqref{erm_samp} has the same expected value (over randomness in $\Dcal$ and $\Dcal'$) as $\Rewipsaug(\Dcal)$ and can thus serve as an approximation for Equation~\eqref{imp}.

\subsection{Safe Learning by Restricting the Policy Space}
\label{policyres}
As motivated by Theorem~\ref{ipsermexp}, the risk of learning from support deficient data scales with the maximum support divergence $\SDxpol$ among the policies in the policy space $\Pol$. Therefore, our third approach restricts the policy space to the subset $\Pol^{\kappa} \subset \Pol$ that contains the policies $\pol \in \Pol$ with an acceptably low support divergence $\SDxpol \le \kappa$.
\begin{equation}
	\Pol^{\kappa} =	\left\{\pol \in \Pol| \SDxpol \leq \kappa\right\}
\end{equation}
The parameter $\kappa$ has an intuitive meaning. It specifies the maximum probability mass that a learned policy can place on unsupported actions. Typically the choice of $\kappa$ is application dependent, limiting the maximum bias of the ERM procedure according to Proposition~\ref{ipsregbias} while not explicitly torquing the rewards like in conservative reward imputation. 
A key challenge, however, is implementing this restriction of the hypothesis space, such that the ERM learner $$\hat{\pol} = \argmax_{\pol \in \Pol^\kappa} [ \Rewips(\pol) ]$$ only considers the subset $\Pol^{\kappa} \subset \Pol$.

In particular, we do not have access to the context distribution $P(\Xcal)$ for calculating $\SDxpol$, nor would it be possible to enumerate all $\pol \in \Pol$ to check the condition $\SDxpol \le \kappa$, which itself requires a possibly infeasible iteration over all actions. The following theorem gives us an efficient way of estimating and controlling $\SDxpol$ without explicit knowledge of $P(\Xcal)$ or access to the logging policy $\polnull$ beyond the logged propensities.

\begin{restatable}{thm}{cvsd}
For contexts $x_i$ drawn i.i.d from $P(\Xcal)$, action $y_i$ drawn from logging policy $\pi_0(\Ycal|x_i)$, we define $S_{\Dcal}(\pol|\polnull) = \frac{1}{n}\sum_{i=1}^n \frac{\pol(\y_i|\x_i)}{\polnull(\y_i|\x_i)}$. For any policy $\pi$ it holds that
\begin{equation}
\EE_{\x\sim P(\Xcal)}\EE_{\y\sim\polnull(\Ycal|\x)}[S_{\Dcal}(\pi|\pi_0)] + \SDxpol = 1	
\end{equation}
\end{restatable}

The proof is shown in Appendix~\ref{Appendix:theo2}. Using this theorem, the following proposition 
gives us an efficient way of implementing the constraint $\SDxpol \le \kappa$ via $1-S_{\Dcal}(\pol|\polnull)$.

\begin{restatable}{propos}{cvest}
\label{cvest}
    For any given $\kappa\in(0,1)$, and for $\epsilon$ with $0<\epsilon<\kappa/2$, let $p_{min}$ denote the minimum propensity on the supported set with $p_{min}=min_{x,y\in\Ucal(x,\pi_0)^c}\pi_0(y|x)$, then with probability larger than $1-2\exp(-2n\epsilon^2 p^2_{min})$, the constraint $1-\kappa+\epsilon\leq S_{\Dcal}(\pol|\polnull) \leq 1-\epsilon$ will ensure $0\leq \SDxpol\leq \kappa$.
\end{restatable}

The proof is provided in Appendix~\ref{Appendix:prop3}. We can thus use $1-S_{\Dcal}(\pol|\polnull)$ as a surrogate for $\SDxpol$ in the IPS training objective (or similar objectives like DR, clipped IPS, or CAB). 
\begin{equation} 
\begin{split}
\label{polressur}
	&\argmin_{\pi_w \in \Pol} \frac{1}{n}\sum_{i=1}^n \frac{\pi_w(y_i|x_i)}{\pi_0(y_i|x_i)}r_i \\
	&\text{    subject to    } 1-\kappa+\epsilon
	 \leq \frac{1}{n}\sum_{i=1}^n \frac{\pi_w(y_i|x_i)}{\pi_0(y_i|x_i)}
	 \leq 1-\epsilon
\end{split}
\end{equation}

Using Lagrange multipliers, an equivalent dual form of Equation~\eqref{polressur} is:
\begin{equation}
\label{est: policy_res}
\max_{u_1, u_2\geq 0}\min_{\pi_w\in\Pol}	\frac{1}{n}\sum_{i=1}^n \frac{\pi_w(y_i|x_i)}{\pi_0(y_i|x_i)} (r_i+u_1-u_2)-u_1(1-\epsilon)+u_2(1-\kappa+\epsilon) 
\end{equation}
For each fixed $(u_1, u_2)$ pair, the inner minimization objective is ERM with  IPS where the reward is shifted by $k=(u_1-u_2)$. So, we can select $k$, solve \eqref{est: policy_res}, and compute the corresponding $\kappa$ afterwards. To achieve a desired $\kappa$, we can use any suitable root finding method for $k$. We simply perform a grid search over $k=u_1-u_2$.

\begin{figure*}[tb]
\centering
  \includegraphics[width=18cm,height=9cm]{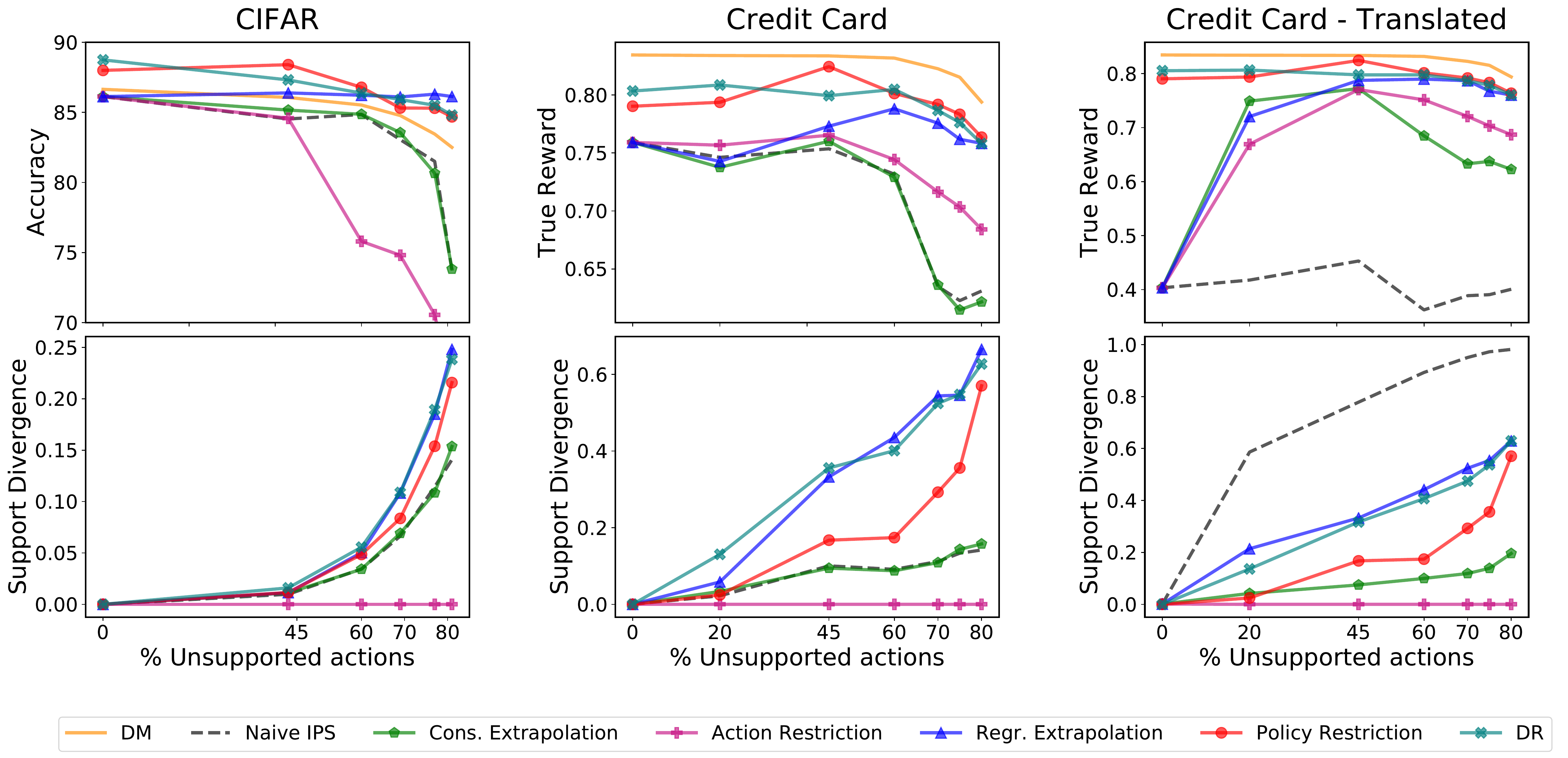} 
    \caption{Test set accuracy and support divergence of the learned policies as we pick logging policies that have increasing fractions of unsupported actions.}
  \label{tao_vs_oracle}
\end{figure*}

\paragraph{Empirical Model Selection for $k$.} While we may have a desired risk tolerance $\kappa$ for selecting $k$ in some applications, on others we may want to select the $k$ that maximize performance on a validation set. This again requires some estimate of the reward on the unsupported actions. We thus explore Conservative Extrapolation, DM, and the following model-independent approach we call \emph{\ms}. 
\ms\ aims to get the best model-free estimate of the reward on the unsupported-action set that the available data admits. In particular, we construct a minimally supported policy $\pi_{MinSup}$ that is closest to a policy that only takes unsupported actions while still having full support. The construction is as follows: for each context $x$, we greedily put all the probability $p$ on the action that has the lowest propensity, while keeping the IPS weights ($\frac{p}{\pi_0(y|x_i)}$) to be bounded by 100. If $p\leq 1$, we continue this procedure with the next-lowest propensity until we have distributed all of the probability mass. We then estimate the value of the constructed policy $\pi_{MinSup}$ using IPS to arrive at the following estimator for any target policy $\pi$, which substitutes $\Rewips(\pi_{\ms})$ for the missing support of $\Rewips(\pi)$. 
\begin{equation}
    \Rminsup(\pi) = \Rewips(\pi)+\Big(1-S_{\Dcal}(\pol|\polnull)\Big)\Rewips(\pi_{\ms})
\end{equation}
Note that IPS is unbiased for $\pi_{MinSup}$ since it has the same support set as $\pi_0$. Furthermore, it has bounded variance since the IPS weights are bounded by construction. For all the empirical evaluations in this paper, we use \ms\ to select the optimal $k$. We also compare \ms\ to DM and Conservative Extrapolation for this model-selection problem in Section~\ref{modsel}. 

\paragraph{Practical Considerations.} Among the methods we proposed for dealing with support deficiency, the Policy Restriction approach is easy to implement, does not require an additional regression model with unknown bias, and it does not require access to the logging policy during training or testing. In particular, the form of the inner objective coincides with that of BanditNet \citep{Joachims/etal/18a}, which is known to work well for deep network training by controlling propensity overfitting \citep{Swaminathan/Joachims/15c}.
 
\section{Empirical Evaluation} 
\label{experiments}
We empirically analyze and compare the effectiveness and robustness of the three approaches: restricting the action space, reward extrapolation, and restricting the policy space. We use two real-world datasets, namely the image-classification dataset CIFAR10 \citep{krizhevsky2014cifar} and the credit-card fraud dataset of \cite{dal2015calibrating}, from which we generate various degrees of support deficient bandit data. 

The experiments are set up as follows. We first create a train-validation-test split for both datasets. The training set is used to generate bandit datasets for learning, the validation set is used to generate bandit datasets for model selection, and the full-information test set serves as ground truth for evaluating the learned policies. To simulate bandit feedback for the CIFAR10 dataset, our experiment setup follows traditional supervised $\to$ bandit conversion for multi-class classification datasets \citep{beygelzimer2009offset}.
To not limit our evaluation to binary multi-class rewards, we choose a different methodology for the credit-card dataset by designating some features as corresponding to actions and rewards. More details are given in Appendix~\ref{Appendix:experiment}.

For both logging and target policies, we train softmax policies (Equation~\eqref{softmax_policy}) where $f_w(x,y)$ is a neural network. We use the ResNet20 architecture \citep{he2016deep} for CIFAR10, and a fully connected 2-layer network for the credit-card dataset. We then introduce a temperature parameter $\tau$ into the learned policy via $\tau f_w(x,y)$ to be able to control its stochasticity and support deficiency.  
In particular, we enforce zero support for some actions by clipping the propensities to 0 if they are below a threshold of $\epsilon=0.01$. The larger $\tau$, the higher the support deficiency. Note that setting the threshold at $\epsilon=0.01$ allows us to control support 
without having to worry about variance control. More details are given in Appendix~\ref{Appendix:experiment}.

The estimators that we examined in our experiments are: first, traditional baselines including naive IPS (Eq.~\eqref{est:ips}) and DM; second, the Action Restriction approach (Eq.~\eqref{est:action_res}); third, three Reward Extrapolation methods (Conservative Extrapolation, Regression Extrapolation, and DR (Eq.~\eqref{est:augdr}); as well as, fourth, the Policy Restriction approach (Eq.~\eqref{est: policy_res}). If not mentioned otherwise, \ms\ is used to select the $k$ parameter and all experiment results are averaged over 5 runs.

\begin{figure}[tb]
  \centering
  \begin{minipage}[b]{0.48\textwidth}
    \includegraphics[width=8.4cm,height=5cm]{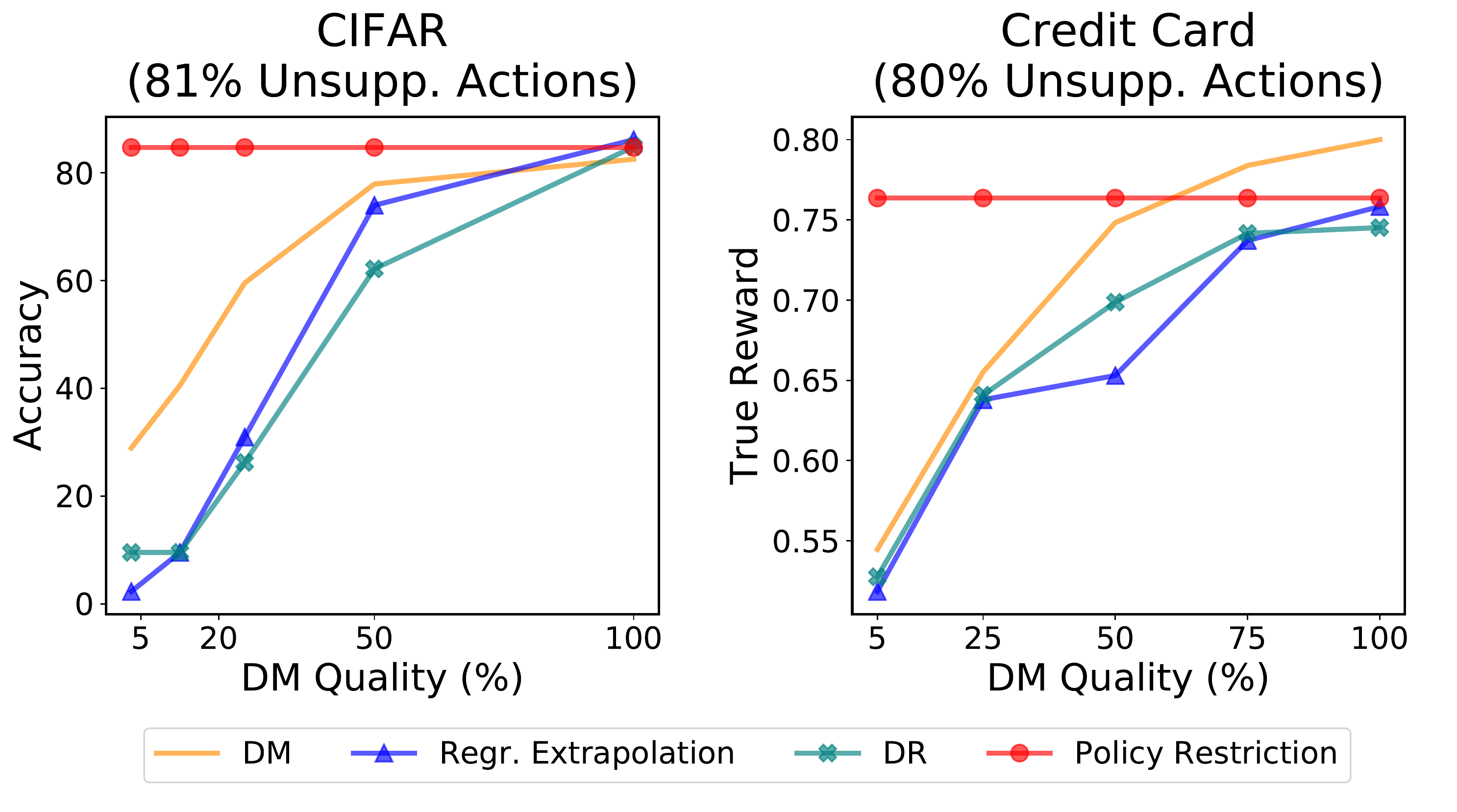} 
    \caption{Test set accuracy for different levels of misspecification of the regression model.}
    \label{partial_dm_cifar}
  \end{minipage}
\end{figure}

\begin{figure*}[t]
\centering
  \includegraphics[width=18cm, height=6cm]{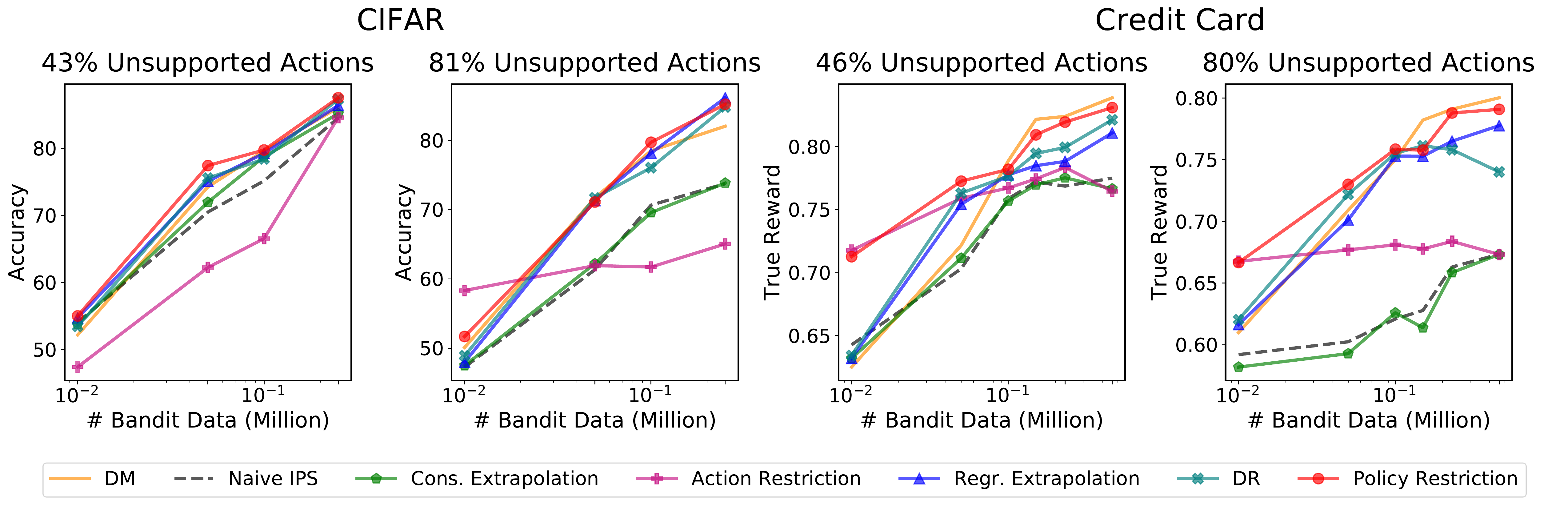}
  \caption{Test set accuracy as the amount of training data increases for two levels of support deficiency.} 
  \label{data_vs_oracle}
\end{figure*}

\begin{figure}[t]
  \centering
  \begin{minipage}[b]{0.48\textwidth}
    \includegraphics[width=8.4cm,height=5cm]{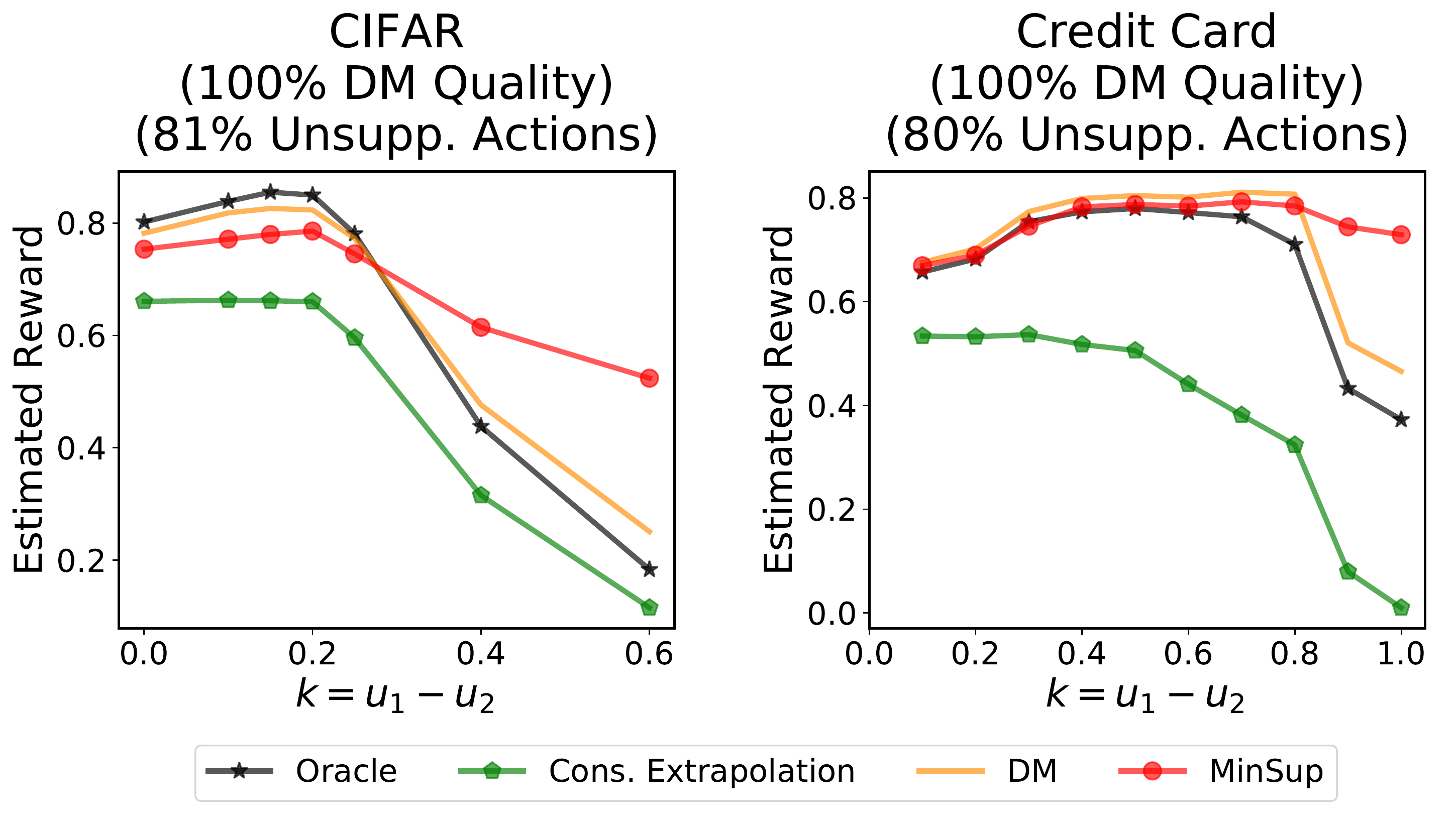}
    \caption{Comparison of estimators of validation set accuracy for the policies learned by Policy Restriction when the parameter $k$ is varied.}
    \label{msplot}
  \end{minipage}
\end{figure}

\begin{figure*}[tbp]
\centering
  \includegraphics[width=18cm, height=5.6cm]{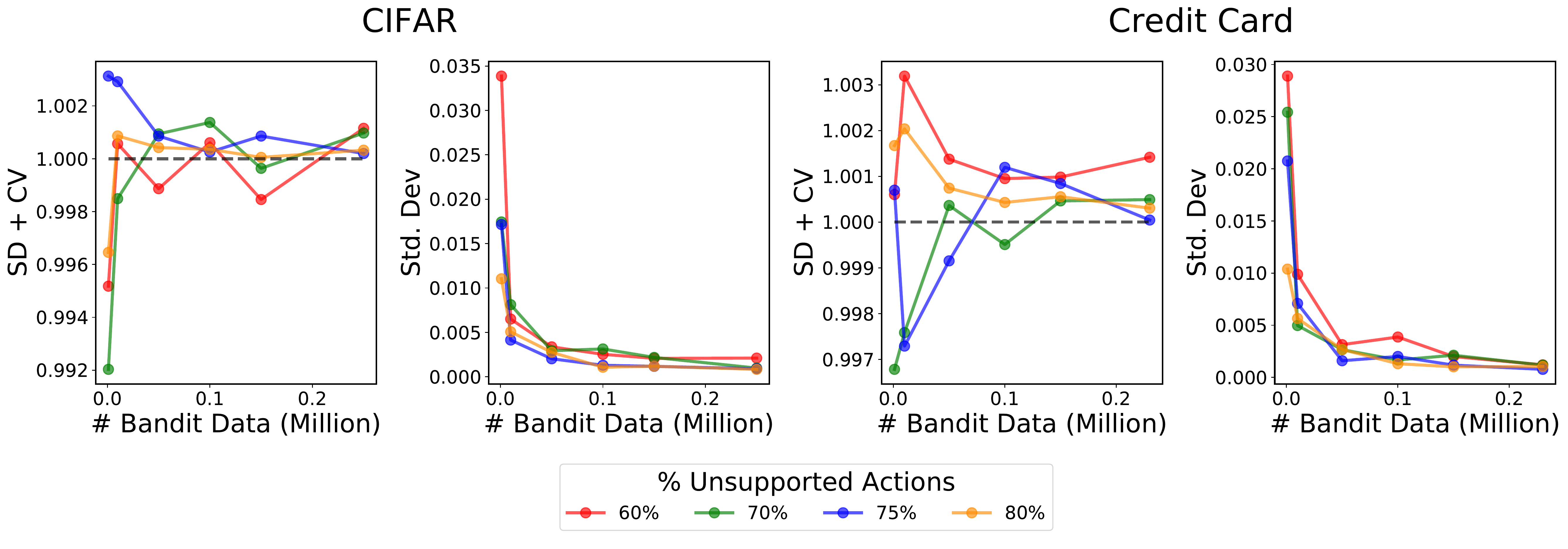}
  \caption{Estimate of $S_{\Dcal}(\pi|\pi_0)+\Dcal_{\Xcal}(\pi|\pi_0)$ and its estimated standard deviation as the amount of training data increases.}
\label{cv_vp}
\end{figure*}


        
        

\subsection{Experiments and Findings}
\label{sec: results}

The following experiments investigate the key properties of the estimators, and they inform our recommendations and conclusions.

\paragraph{How do the methods perform at different level of support deficiency?}

Figure \ref{tao_vs_oracle} shows the test accuracy and support divergence $\SDxpol$ as support deficiency increases. First, as expected, learning using naive IPS degrades on both datasets. Note that naive IPS coincides with Conservative Extrapolation in the left two columns, since both datasets are scaled to have a minimum reward of zero. In the rightmost column, however, we translated the rewards from $[0,1]$ to $[-1,0]$. This has a strong detrimental effect on naive IPS. IPS is inherently imputing reward 0, and the performance of naive IPS highly depends on the position of $0$ in the range of the reward.
Second, the Action Restriction approach also performs poorly. While its support divergence $\SDxpol$ is zero and thus bias is not the problem, we conjecture that the best actions are often pruned from the action-restricted policy space $\Polres$.
Third, Regression Extrapolation tends to perform better than Conservative Extrapolation in our experiments. On both datasets, the DM model turns out to be quite good, which also benefits DR. However, on the credit-card dataset the regression seems better at ranking than at predicting the true reward, which explains why DM performs better than Regression Extrapolation. 
Fourth, the methods that performs well on both datasets are Policy Restriction and DR. Unlike all the other IPS-based methods, Policy Restriction performs well even under the translated rewards in the third column of Figure \ref{tao_vs_oracle}. This is because the objective of Policy Restriction coincides with that of BanditNet \citep{Joachims/etal/18a}, which is known to remedy propensity overfitting due to the lack of equivariance of the IPS estimator \citep{Swaminathan/Joachims/15d}.

\paragraph{How does regression-model misspecification affect the estimators?} While DR performed well in the previous experiments, it has no mechanism for guarding against bias from model misspecification under support deficiency. The same limitation also applies to the other estimators that rely on regression imputation for the unsupported actions. We thus examine how different estimators deal with model misspecification. To simulate increasing levels of model misspecification, we train the regression function on subsets of input features of varying size. Results are shown in Figure~\ref{partial_dm_cifar}. As the number of input feature decreases and misspecification thus increases, the effectiveness of DR, Regression Extrapolation and DM decreases substantially. Since Policy Restriction does not rely on any regression model, it is unaffected.

\paragraph{How does the learning performance change with the amount of training data?}

Figure~\ref{data_vs_oracle} shows how accuracy on the test set changes with the amount for training data for two levels of support deficiency. 
Policy Restriction is at least competitive with Regression Extrapolation, DR, and DM over most of the range. Action Restriction can take the least advantage of more data. This is plausible, since its maximum performance is limited by the available actions. For similar reasons, Conservative Extrapolation, and equivalently IPS, also flatten out, since they also tightly restrict the action space by imputing the minimum reward.

\begin{table}[tb]
\centering
  \begin{tabular}[b]{  c  c  c  c  c  }\hline
        \% Unsupp. & Oracle & DM & Cons. Extra. & \ms \\ \hline
        43 & 88.397 & 87.526 & \textbf{88.397} & \textbf{88.397} \\
        60 & 86.782 & \textbf{86.782}  & \textbf{86.782} & \textbf{86.782}\\
        69 & 86.462 & 85.308 & 85.308 & 85.308 \\ 
        77 & 85.295 & 85.282 & 84.090 & \textbf{85.295} \\ 
        81 & 85.192 & \textbf{85.192} & 83.526 & 84.680 \\ \hline
    \end{tabular}
\caption{Test error rates on the CIFAR10 data using the respective model selection method under varying levels of support deficiency in the logging policy.}
\label{tabel:model_selection_cifar}
\vspace{-0.5cm}
\end{table}


\paragraph{How effective is model selection for $k$?} 
\label{modsel}
If there is no pre-specified risk tolerance $\tau$ that is derived from application requirements, Section~\ref{policyres} proposed to select the parameter $k$ on a validation set. Table~\ref{tabel:model_selection_cifar} shows how the proposed \ms\ model selection criterion compares against model selection via DM and Conservative Extrapolation. We also report the skyline performance of an  Oracle model selector which has access to the full-information validation set. Conservative extrapolation consistently underperforms in high-deficiency settings, while both \ms\ and DM perform well as model selection criteria. This is explained by the plots in Figure~\ref{msplot}, which show the value of the validation set estimates. While \ms\ and DM have their optimum close to that of Oracle, the Conservative Extrapolation curve is substantially off. Additional results in Appendix~\ref{Appendix:results}, however, show that model misspecification can again substantially affect the performance of model selection via DM.



\paragraph{How accurate is the estimator of support divergence?}
Policy restriction relies on the estimated support divergence, and we now evaluate the effectiveness of this estimator. The target policy is the uniform policy while the logging policies are varying in their support deficiency. We investigate the behaviour of $S_{\Dcal}(\pi|\pi_0)+\Dcal_{\Xcal}(\pi|\pi_0)$ for increasing amounts of training data and different support deficiency of the corresponding logging policy. Results are shown in Figure~\ref{cv_vp} averaged over 10 runs. The figure shows that the sum converges to 1 and the variance of the estimate decreases as the amount of training data increases. The curves for different levels of support deficiency converge in a similar fashion and we conjecture it is due to the effect of clipping the propensity at the same threshold $\epsilon=0.01$, which results in $p_{min}=0.01$ in the bound shown in Proposition~\ref{cvest}.

\section{Discussion and Recommendations}

We now discuss the advantages and disadvantages of the various approaches to dealing with support deficiency, and provide guidance for their use in practice.

While Action Restriction may be the the most natural and direct approach at first glance, we find that it has substantial drawbacks. In particular, it is computationally expensive since we need to calculate the propensity under the logging policy $\pi_0(y|x)$ for every $x$ not only at training time, but also at testing time. This is particularly problematic when policies are updated in a frequent manner, since we need to revisit the whole sequence of past executed logging policies at test time. Furthermore, its learning performance is substantially worse than other methods, since actions are limited to an overly conservative regime that enforces zero support divergence.  

Conservative Extrapolation allows the target policy to select actions that have zero support under the logging policy, and performance tends to be better than for Action Restriction. However, Conservative Extrapolation typically performs worse than Policy Restriction, Regression Extrapolation, and DR. All methods are more efficient that Action Restriction at test time as they do not require evaluating the old logging policy. During training, however, all Reward Imputation methods (i.e. Conservative Extrapolation, Regression Extrapolation, and DR) need to evaluate all actions, which can be expensive but ameliorated through sampling. A key risk of both Regression Extrapolation and DR is that they rely on a regression model, which can introduce biases from model misspecification that are fundamentally unknown. The estimators provide no mechanism for guarding against such biases.

Policy Restriction does not rely on a regression model, which eliminates the need for training such a model. Furthermore, it is the only method that allows flexible risk control through the parameter $\kappa$ or $k$ respectively. If the application does not provide a risk threshold, $k$ can be selected empirically via \ms. From a computational perspective, Policy Restriction is efficient at both training and test time, and it has minimal logging requirements as it only requires the logged propensity of the chosen action. Since it consistently showed at least competitive generalization performance across a wide range of settings, we conclude that it is a preferable choice for practical applications -- especially when there are concerns about model misspecification.

\section{Conclusions}
This paper presented the first comprehensive analysis of support deficiency in off-policy learning for contextual bandits. In particular, it identified and explored three approaches to dealing with support deficiency: restricting the action space, reward extrapolation, and restricting the policy space. The paper characterized the theoretical properties of these approaches, and empirically evaluated their performance and robustness. We conclude that restricting the policy space is particularly effective, since it provides explicit risk control, performs well in terms of learning performance, and it is easy and efficient to implement.



\begin{acks}
This research was supported in part by NSF Award IIS-1901168 and by a Bloomberg Fellowship. All content represents the opinion of the authors, which is not necessarily shared or endorsed by their respective employers and/or sponsors.
\end{acks}

\bibliographystyle{ACM-Reference-Format}
\balance
\bibliography{iclr2020_conference}
\newpage

\newpage
\appendix

\section{Appendix: Proofs}
\label{Appendix:proofs}
\subsection{Proof of Proposition 1}
\label{Appendix: proofapp1}
\ipsbias*
\begin{proof}
Recall $\delta(x,y)=\EE_r[r(x,y)|x,y]$, and logged data $\Dcal\sim \Pcal_{\Xcal}\times\pi_0(\cdot|\Xcal)\times\Pcal_{r}$.
\begin{equation}
\begin{split}
    bias(\Rewips(\pi)) & = \EE_{\Dcal}\left[\hat{R}_{IPS}(\pi)\right]- R(\pi)\\
    &= \EE_x\bigg[\sum_{y\in (\Ucal(x,\pi_0))^c}\pi_0(y|x)\frac{\pi(y|x)}{\pi_0(y|x)}\delta(x,y)\\
    &-\sum_{y\in\Ycal}\pi(y|x)\delta(x,y)\bigg]\\
    &= \EE_{x}\bigg[-\sum_{y\in \Ucal(x,\pi_0)}\pi(y|x)\delta(x,y)\bigg]
\end{split}
\end{equation}
\end{proof}

\subsection{Proof of Proposition 2}
\label{Appendix: proofapp2}
\ipsregbias*
\begin{proof}
Based on the definition of $\Rewipsaug(\pol)$:
\begin{equation}
    \begin{split}
        & \EE_{\Dcal}[\Rewipsaug(\pol)]-R(\pi)\\
        = & \EE_x\bigg[\sum_{y\in \Ucal(x,\pi_0)^c}\pi(y|x)\delta(x,y)+\sum_{y\in \Ucal(x,\pi_0)}\pi(y|x)\hat{\delta}(x,y)\bigg]-R(\pi)\\
        = &\EE_x\bigg[\sum_{y\in\Ucal(x,\pi_0)}\pi(y|x)(\hat{\delta}(x,y)-\delta(x,y))\bigg]\\
        = &\EE_x\bigg[\sum_{y\in\Ucal(x,\pi_0)}\pi(y|x)\Delta(x,y)\bigg]
    \end{split}
\end{equation}
The second equality is based on the decomposition of $R(\pi)$, and the last one is based on the definition of $\Delta(x,y):=\hat{\delta}(x,y)-\delta(x,y)$ for all $x\in\Xcal,y\in\Ycal$.
\end{proof}

\subsection{Proof of Efficient Approximation}
\label{Appendix:effapp}
\begin{claim}
The empirical risk defined by in Equation~\eqref{erm_samp} has the same expectation (over randomness in $\Dcal$ and sampling) as $\Rewipsaug(\Dcal)$.
\end{claim}
\begin{proof}
Taking the expectation of empirical risk defined in Equation~\eqref{erm_samp}:
\begin{equation*}
\label{regexp}
\begin{split}
&\EE\bigg[\frac{1}{n}\sum_{i=1}^n \frac{\pol(\y_i|\x_i)}{\polnull(\y_i|\x_i)}\rew_i +\frac{1}{m}\sum_{j=1}^m \frac{\pol(\y_j|\x_j)}{\prop_j}\rewhatxy(\x_j,\y_j)\bigg]\\
\end{split}
\end{equation*}
\begin{equation}
\begin{split}
&=\EE_x\bigg[\sum_{y\in \Ucal(x,\pi_0)^c}\pi_0(y|x)\frac{\pi(y|x)}{\pi_0(y|x)}\delta(x,y)\bigg]\\
&+\EE_x\bigg[\sum_{y\in\Ucal(x,\pi_0)}\frac{1}{|\Ucal(x,\pi_0)|}\frac{\pi(y|x)}{\frac{1}{|\Ucal(x,\pi_0)|}}\hat{\delta}(x,y)\bigg]\\
&=\EE_x\bigg[\sum_{y\in \Ucal(x,\pi_0)^c}\pi(y|x)\delta(x,y)\big] + \EE_x\big[\sum_{y\in \Ucal(x,\pi_0)}\pi(y|x)\hat{\delta}(x,y)\bigg]
\end{split}
\end{equation}
Now we will show it has the same expectation with $\Rewipsaug(\pol)$
\begin{equation}
\label{regsamexp}
\begin{split}
&\EE_{\Dcal}\bigg[\frac{\pol(\y_i|\x_i)}{\polnull(\y_i|\x_i)}\rew_i+\sum_{\y\in\Ucalpol{\x_i}}\pol(\y|\x_i)\rewhatxy(\x_i,\y)\bigg]\\
&=\EE_x\bigg[\EE_{y\sim\pi_0}\Big[\frac{\pi(y|x)}{\pi_0(y|x)}\delta(x,y)\Big]+\sum_{y'\in\Ucal(x,\pi_0)}\pi(y'|x)\hat{\delta}(x,y')\bigg]\\
&=\EE_{x}\bigg[\sum_{y\in \Ucal(x,\pi_0)^c}\pi_0(y|x)\frac{\pi(y|x)}{\pi_0(y|x)}\delta(x,y)+\sum_{y'\in\Ucal(x,\pi_0)}\pi(y'|x)\hat{\delta}(x,y')\bigg]\\
&=\EE_x\bigg[\sum_{y\in \Ucal(x,\pi_0)^c}\pi(y|x)\delta(x,y)\bigg] + \EE_x\bigg[\sum_{y\in \Ucal(x,\pi_0)}\pi(y|x)\hat{\delta}(x,y)\bigg]
\end{split}
\end{equation}
The proof is done by comparing Equation~\eqref{regexp} and Equation~\eqref{regsamexp}.
\end{proof}

\subsection{Proof of Theorem 2}
\label{Appendix:theo2}
\cvsd*
\begin{proof}
\begin{equation}
\begin{split}
    &\EE_{\x\sim P(\Xcal)}\EE_{\y\sim\polnull(\cdot|\x)}[S_{\Dcal}(\pi|\pi_0)] + \SDxpol \\
    &=\EE_x\bigg[\sum_{y\in \Ucal(x,\pi_0)^c}\pi_0(y|x)\frac{\pi(y|x)}{\pi_0(y|x)}\bigg] +\Dcal_{\Xcal}(\pi|\pi_0)\\
    &=\EE_x\bigg[\sum_{y\in\Ucal(x,\pi_0)^c}\pi(y|x)]+\EE_x[\sum_{y\in\Ucal(x,\pi_0)}\pi(y|x)\bigg]\\
    &=\EE_x\bigg[\sum_{y\in\Ycal}\pi(y|x)\bigg]=1
\end{split}
\end{equation}

The first equality is based on definition of $S_{\Dcal}(\pi|\pi_0)$ and the second equality is based on definition of support divergence.
\end{proof}

\subsection{Proof of Proposition 3}
\label{Appendix:prop3}
\cvest*
\begin{proof}
Recall $S_{\Dcal}(\pi|\pi_0)=\frac{1}{n}\sum_{i=1}^n \frac{\pi(y_i|x_i)}{\pi_0(y_i|x_i)}$ with $(x_i,y_i)$ draw i.i.d from $P(\Xcal)\times\pi_0(\Ycal|x)$. Also, it is easy to see $\EE_{x,y\sim\pi_0(\cdot|x)}[\frac{\pi(y|x)}{\pi_0(y|x)}]=1-\Dcal_{\Xcal}(\pi|\pi_0)$. Let $p_{min}$ denote the smallest propensity under supported action set with $p_{min}:=\min_{x,y\in\Ucal(x,\pi_0)^c}\pi_0(y|x)>0$, then the random variable $\frac{\pi(y|x)}{\pi_0(y|x)}$ is strictly bounded between $[0, \frac{1}{p_{min}}]$. Applying Hoeffding's bound gives:
\begin{equation}
\begin{split}
    \PP(\Dcal_{\Xcal}(\pi|\pi_0)<1-S_{\Dcal}(\pi|\pi_0)-\epsilon)&=\PP(S_{\Dcal}(\pi|\pi_0)\\&-(1-\Dcal_{\Xcal}(\pi|\pi_0))<-\epsilon)\\
    &\leq exp(-2n\epsilon^2p^2_{min})
\end{split}
\end{equation}
Since $S_{\Dcal(\pi|\pi_0)}\leq 1-\epsilon$ gives $1-S_{\Dcal}(\pi|\pi_0)-\epsilon\geq 0$, then we have
\begin{equation}
    \PP(\Dcal_{\Xcal}(\pi|\pi_0)<0)\leq exp(-2n\epsilon^2p^2_{min})
\end{equation}
Similar for the other direction, Hoeffding's bound gives:
\begin{equation}
\begin{split}
    \PP(\Dcal_{\Xcal}(\pi|\pi_0)>1-S_{\Dcal}(\pi|\pi_0)+\epsilon)&=\PP(S_{\Dcal}(\pi|\pi_0)-(1-\Dcal_{\Xcal}(\pi|\pi_0))\\
    &>\epsilon)\leq exp(-2n\epsilon^2p^2_{min})
\end{split}
\end{equation}
Since $S_{\Dcal(\pi|\pi_0)}\geq 1+\epsilon-\kappa$ gives $1-S_{\Dcal}(\pi|\pi_0)+\epsilon\leq \kappa$, then we have
\begin{equation}
    \PP(\Dcal_{\Xcal}(\pi|\pi_0)\geq\kappa)\leq exp(-2n\epsilon^2p^2_{min})
\end{equation}
Combining the above, we have
\begin{equation}
\begin{split}
     \PP(0\leq\Dcal_{\Xcal}(\pi|\pi_0)\leq\kappa)&= 1- \PP(\Dcal_{\Xcal}(\pi|\pi_0)<0)-\PP(\Dcal_{\Xcal}(\pi|\pi_0)>\kappa)\\
     &\geq 1-2exp(-2n\epsilon^2p^2_{min})
\end{split}
\end{equation}
\end{proof}

\section{Appendix: Experiment Details}
\label{Appendix:experiment}
\paragraph{Datasets and baseline.} 
We follow a 75:10:15 train-validation-test split for credit card fraud detection dataset, while for CIFAR10 already coming with a train-test split, we keep 10\% of the training set as validation set. Baseline estimators are IPS and DM, the hyperparameters (learning rate, $L_2$ regularization) are optimized for all the methods based on the validation set.

\paragraph{Bandit data generation.}
For CIFAR10, given supervised data in the format of $\{x_i, y_i^{*}\}_{i=1}^n$ where $x_i$ denotes the 3072 features and $y^{*}_i$ denotes the correct label of data (ranging from $0$ to $9$), under logging policy $\pi_0$, the logged bandit data is generated by drawing $y_i\sim \pi_0(\Ycal|x_i)$, then a deterministic reward is defined as $\ind_{\{y_i=y^{*}_i\}}$. 
For the credit card fraud detection dataset, we aim to have continuous rewards rather than binary. To do so, we throw away the class label and only use the data features for each sample to generate bandit data. To be specific, for each sample with a 28-dimensional feature vector, we define the first 20 features as the contextual information, and use the remaining 8 features as the underlying true reward for 8 different actions (with normalization). 

\paragraph{Logging policy.} 
For CIFAR, we learn the softmax logging policy on 35K full-information data points as a multi-class classification problem with cross-entropy loss. Similar as the experiments on BanditNet \citep{Joachims/etal/18a}, we adopt the conventional ResNet20 architecture but restrict training after a mere two epochs to derive a relative stochastic policy, since it will be easier to add temperature later to control its stochasticity and support deficiency. Similarly, for the credit card fraud detection dataset, the softmax logging policy is learned on 8K full-information data points by treating it as a multi-class classification problem using cross-entropy loss and the label being the action with the highest reward on this specific context. For CIFAR, the logging policy we trained has a 57.43\% accuracy on the test-set; whereas for the credit card fraud detection dataset, the logging policy has an expected true reward of 0.71.

\paragraph{Reward estimator.}
For each experiment, we train a different regression function using the full bandit dataset. We use the same architecture as the one used for off-policy learning (ResNet20 for CIFAR10, two layer neural network for the Credit Card dataset) - where the final layer is the size of the actions, specifying the reward for each action given a particular context. The regression function is trained using the MSE objective.

\label{Appendix:results}
\paragraph{Model selection details.}
We provide the model selection result for credit card in Table~\ref{tabel:model_selection_credit}. In this dataset, both DM and \ms ~achieve near-oracle performance under different levels of support deficiency. In Figure~\ref{ms_cifar}, we test how various estimators perform under model-misspecification on the CIFAR dataset. As the number of input feature decreases. the quality of DM diminishes, which affects its performance used in model selection. \ms ~is pretty robust under model misspecification.
\begin{table}[!htb]
\centering
  \begin{tabular}[b]{  c  c  c  c  c  }\hline
        \% Unsupp. & Oracle & DM  & Cons. Extra. & \ms\\ \hline
            0 & 0.793 & \textbf{0.793} & \textbf{0.793} & \textbf{0.793} \\ \hline
            20 & 0.803 & 0.791  & 0.791 & \textbf{0.803}\\ \hline
            45 & 0.799 & 0.798  & 0.795 & \textbf{0.799}\\ \hline
            60 & 0.797 & \textbf{0.797} & \textbf{0.797} & 0.796\\ \hline
            75 & 0.778 & \textbf{0.778} & 0.766 & \textbf{0.778}\\ \hline
            80 & 0.774 & 0.759 & 0.751 & 0.759\\ \hline
    \end{tabular}
\caption{Model selection results for Credit Card with varying levels of support deficiency in the logging policy.}
\label{tabel:model_selection_credit}
\vspace{-0.4cm}
\end{table}

\begin{figure}[!htb]
 \includegraphics[width=8cm, height=6.5cm]{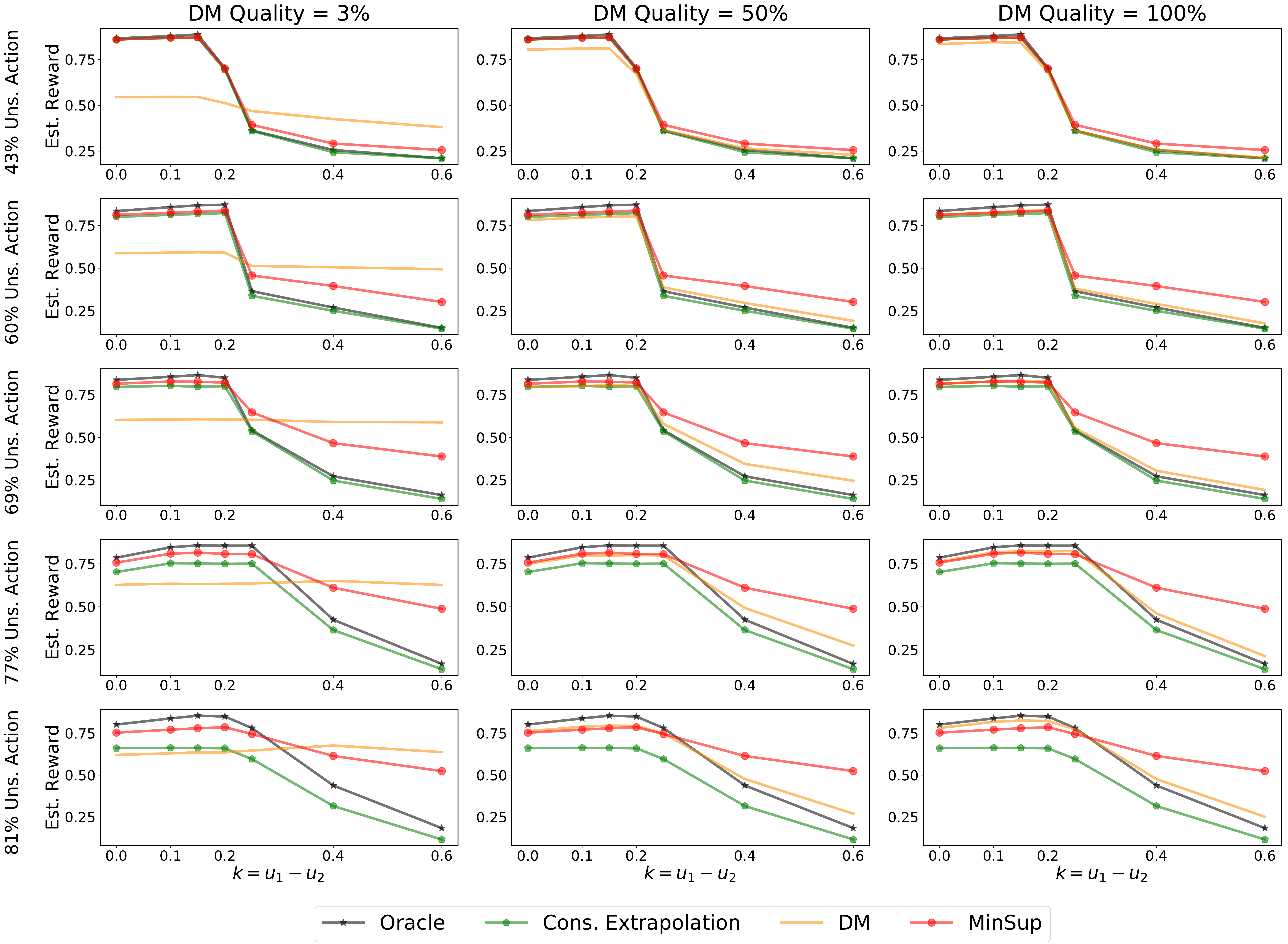} 
\caption{Model selection results for various support deficiencies and different levels of model misspecifications on CIFAR.} 
\label{ms_cifar}
\end{figure}

\end{document}

%% file: math_commands.tex
\usepackage{amsmath,amsfonts,bm}

\def\eqref#1{(\ref{#1})}

\def\1{\bm{1}}

\DeclareMathAlphabet{\mathsfit}{\encodingdefault}{\sfdefault}{m}{sl}
\SetMathAlphabet{\mathsfit}{bold}{\encodingdefault}{\sfdefault}{bx}{n}

\DeclareMathOperator*{\argmax}{arg\,max}
\DeclareMathOperator*{\argmin}{arg\,min}

%% file: sample-sigconf.bbl

\begin{thebibliography}{29}


\ifx \showCODEN    \undefined \def \showCODEN     #1{\unskip}     \fi
\ifx \showDOI      \undefined \def \showDOI       #1{#1}\fi
\ifx \showISBNx    \undefined \def \showISBNx     #1{\unskip}     \fi
\ifx \showISBNxiii \undefined \def \showISBNxiii  #1{\unskip}     \fi
\ifx \showISSN     \undefined \def \showISSN      #1{\unskip}     \fi
\ifx \showLCCN     \undefined \def \showLCCN      #1{\unskip}     \fi
\ifx \shownote     \undefined \def \shownote      #1{#1}          \fi
\ifx \showarticletitle \undefined \def \showarticletitle #1{#1}   \fi
\ifx \showURL      \undefined \def \showURL       {\relax}        \fi
\providecommand\bibfield[2]{#2}
\providecommand\bibinfo[2]{#2}
\providecommand\natexlab[1]{#1}
\providecommand\showeprint[2][]{arXiv:#2}

\bibitem[\protect\citeauthoryear{Agarwal, Hsu, Kale, Langford, Li, and
  Schapire}{Agarwal et~al\mbox{.}}{2014}]%
        {agarwal2014taming}
\bibfield{author}{\bibinfo{person}{Alekh Agarwal}, \bibinfo{person}{Daniel
  Hsu}, \bibinfo{person}{Satyen Kale}, \bibinfo{person}{John Langford},
  \bibinfo{person}{Lihong Li}, {and} \bibinfo{person}{Robert Schapire}.}
  \bibinfo{year}{2014}\natexlab{}.
\newblock \showarticletitle{Taming the monster: A fast and simple algorithm for
  contextual bandits}. In \bibinfo{booktitle}{\emph{ICML}}.
\newblock


\bibitem[\protect\citeauthoryear{Beygelzimer and Langford}{Beygelzimer and
  Langford}{2009}]%
        {beygelzimer2009offset}
\bibfield{author}{\bibinfo{person}{Alina Beygelzimer} {and}
  \bibinfo{person}{John Langford}.} \bibinfo{year}{2009}\natexlab{}.
\newblock \showarticletitle{The offset tree for learning with partial labels}.
  In \bibinfo{booktitle}{\emph{KDD}}. ACM, \bibinfo{pages}{129--138}.
\newblock


\bibitem[\protect\citeauthoryear{Bottou, Peters, Qui{\~n}onero-Candela,
  Charles, Chickering, Portugaly, Ray, Simard, and Snelson}{Bottou
  et~al\mbox{.}}{2013}]%
        {bottou2013counterfactual}
\bibfield{author}{\bibinfo{person}{L{\'e}on Bottou}, \bibinfo{person}{Jonas
  Peters}, \bibinfo{person}{Joaquin Qui{\~n}onero-Candela},
  \bibinfo{person}{Denis~X Charles}, \bibinfo{person}{D~Max Chickering},
  \bibinfo{person}{Elon Portugaly}, \bibinfo{person}{Dipankar Ray},
  \bibinfo{person}{Patrice Simard}, {and} \bibinfo{person}{Ed Snelson}.}
  \bibinfo{year}{2013}\natexlab{}.
\newblock \showarticletitle{Counterfactual reasoning and learning systems: The
  example of computational advertising}.
\newblock \bibinfo{journal}{\emph{JMLR}} \bibinfo{volume}{14},
  \bibinfo{number}{1} (\bibinfo{year}{2013}), \bibinfo{pages}{3207--3260}.
\newblock


\bibitem[\protect\citeauthoryear{Dal~Pozzolo, Caelen, Johnson, and
  Bontempi}{Dal~Pozzolo et~al\mbox{.}}{2015}]%
        {dal2015calibrating}
\bibfield{author}{\bibinfo{person}{Andrea Dal~Pozzolo},
  \bibinfo{person}{Olivier Caelen}, \bibinfo{person}{Reid~A Johnson}, {and}
  \bibinfo{person}{Gianluca Bontempi}.} \bibinfo{year}{2015}\natexlab{}.
\newblock \showarticletitle{Calibrating probability with undersampling for
  unbalanced classification}. In \bibinfo{booktitle}{\emph{2015 SSCI}}. IEEE,
  \bibinfo{pages}{159--166}.
\newblock


\bibitem[\protect\citeauthoryear{Dud{\'\i}k, Langford, and Li}{Dud{\'\i}k
  et~al\mbox{.}}{2011}]%
        {dudik2011doubly}
\bibfield{author}{\bibinfo{person}{Miroslav Dud{\'\i}k}, \bibinfo{person}{John
  Langford}, {and} \bibinfo{person}{Lihong Li}.}
  \bibinfo{year}{2011}\natexlab{}.
\newblock \showarticletitle{Doubly Robust Policy Evaluation and Learning}. In
  \bibinfo{booktitle}{\emph{ICML}}.
\newblock


\bibitem[\protect\citeauthoryear{Farajtabar, Chow, and Ghavamzadeh}{Farajtabar
  et~al\mbox{.}}{2018}]%
        {farajtabar2018more}
\bibfield{author}{\bibinfo{person}{Mehrdad Farajtabar}, \bibinfo{person}{Yinlam
  Chow}, {and} \bibinfo{person}{Mohammad Ghavamzadeh}.}
  \bibinfo{year}{2018}\natexlab{}.
\newblock \showarticletitle{More Robust Doubly Robust Off-policy Evaluation}.
  In \bibinfo{booktitle}{\emph{ICML}}. \bibinfo{pages}{1446--1455}.
\newblock


\bibitem[\protect\citeauthoryear{Fujimoto, Meger, and Precup}{Fujimoto
  et~al\mbox{.}}{2018}]%
        {fujimoto2018off}
\bibfield{author}{\bibinfo{person}{Scott Fujimoto}, \bibinfo{person}{David
  Meger}, {and} \bibinfo{person}{Doina Precup}.}
  \bibinfo{year}{2018}\natexlab{}.
\newblock \showarticletitle{Off-policy deep reinforcement learning without
  exploration}.
\newblock \bibinfo{journal}{\emph{arXiv preprint arXiv:1812.02900}}
  (\bibinfo{year}{2018}).
\newblock


\bibitem[\protect\citeauthoryear{Greensmith, Bartlett, and Baxter}{Greensmith
  et~al\mbox{.}}{2004}]%
        {greensmith2004variance}
\bibfield{author}{\bibinfo{person}{Evan Greensmith}, \bibinfo{person}{Peter~L
  Bartlett}, {and} \bibinfo{person}{Jonathan Baxter}.}
  \bibinfo{year}{2004}\natexlab{}.
\newblock \showarticletitle{Variance reduction techniques for gradient
  estimates in reinforcement learning}.
\newblock \bibinfo{journal}{\emph{JMLR}} \bibinfo{volume}{5},
  \bibinfo{number}{Nov} (\bibinfo{year}{2004}), \bibinfo{pages}{1471--1530}.
\newblock


\bibitem[\protect\citeauthoryear{He, Zhang, Ren, and Sun}{He
  et~al\mbox{.}}{2016}]%
        {he2016deep}
\bibfield{author}{\bibinfo{person}{Kaiming He}, \bibinfo{person}{Xiangyu
  Zhang}, \bibinfo{person}{Shaoqing Ren}, {and} \bibinfo{person}{Jian Sun}.}
  \bibinfo{year}{2016}\natexlab{}.
\newblock \showarticletitle{Deep residual learning for image recognition}. In
  \bibinfo{booktitle}{\emph{CVPR}}. \bibinfo{pages}{770--778}.
\newblock


\bibitem[\protect\citeauthoryear{Jiang and Li}{Jiang and Li}{2016}]%
        {jiang2016doubly}
\bibfield{author}{\bibinfo{person}{Nan Jiang} {and} \bibinfo{person}{Lihong
  Li}.} \bibinfo{year}{2016}\natexlab{}.
\newblock \showarticletitle{Doubly Robust Off-policy Value Evaluation for
  Reinforcement Learning}. In \bibinfo{booktitle}{\emph{ICML}}.
  \bibinfo{pages}{652--661}.
\newblock


\bibitem[\protect\citeauthoryear{Joachims, Swaminathan, and de~Rijke}{Joachims
  et~al\mbox{.}}{2018}]%
        {Joachims/etal/18a}
\bibfield{author}{\bibinfo{person}{T. Joachims}, \bibinfo{person}{A.
  Swaminathan}, {and} \bibinfo{person}{M. de Rijke}.}
  \bibinfo{year}{2018}\natexlab{}.
\newblock \showarticletitle{Deep Learning with Logged Bandit Feedback}. In
  \bibinfo{booktitle}{\emph{ICLR}}.
\newblock


\bibitem[\protect\citeauthoryear{Joachims, Swaminathan, and Schnabel}{Joachims
  et~al\mbox{.}}{2017}]%
        {joachims2017unbiased}
\bibfield{author}{\bibinfo{person}{T. Joachims}, \bibinfo{person}{A.
  Swaminathan}, {and} \bibinfo{person}{T. Schnabel}.}
  \bibinfo{year}{2017}\natexlab{}.
\newblock \showarticletitle{Unbiased Learning-to-Rank with Biased Feedback}. In
  \bibinfo{booktitle}{\emph{WSDM}}.
\newblock


\bibitem[\protect\citeauthoryear{Krizhevsky, Nair, and Hinton}{Krizhevsky
  et~al\mbox{.}}{2014}]%
        {krizhevsky2014cifar}
\bibfield{author}{\bibinfo{person}{Alex Krizhevsky}, \bibinfo{person}{Vinod
  Nair}, {and} \bibinfo{person}{Geoffrey Hinton}.}
  \bibinfo{year}{2014}\natexlab{}.
\newblock \showarticletitle{The cifar-10 dataset}.
\newblock \bibinfo{journal}{\emph{online: http://www. cs. toronto.
  edu/kriz/cifar. html}}  \bibinfo{volume}{55} (\bibinfo{year}{2014}).
\newblock


\bibitem[\protect\citeauthoryear{Kumar, Fu, Soh, Tucker, and Levine}{Kumar
  et~al\mbox{.}}{2019}]%
        {kumar2019stabilizing}
\bibfield{author}{\bibinfo{person}{Aviral Kumar}, \bibinfo{person}{Justin Fu},
  \bibinfo{person}{Matthew Soh}, \bibinfo{person}{George Tucker}, {and}
  \bibinfo{person}{Sergey Levine}.} \bibinfo{year}{2019}\natexlab{}.
\newblock \showarticletitle{Stabilizing off-policy q-learning via bootstrapping
  error reduction}. In \bibinfo{booktitle}{\emph{NeurIPS}}.
  \bibinfo{pages}{11761--11771}.
\newblock


\bibitem[\protect\citeauthoryear{Langford and Zhang}{Langford and
  Zhang}{2008}]%
        {langford2008epoch}
\bibfield{author}{\bibinfo{person}{John Langford} {and} \bibinfo{person}{Tong
  Zhang}.} \bibinfo{year}{2008}\natexlab{}.
\newblock \showarticletitle{The epoch-greedy algorithm for multi-armed bandits
  with side information}. In \bibinfo{booktitle}{\emph{NeurIPS}}.
  \bibinfo{pages}{817--824}.
\newblock


\bibitem[\protect\citeauthoryear{Laroche, Trichelair, and Combes}{Laroche
  et~al\mbox{.}}{2017}]%
        {laroche2017safe}
\bibfield{author}{\bibinfo{person}{Romain Laroche}, \bibinfo{person}{Paul
  Trichelair}, {and} \bibinfo{person}{R{\'e}mi Tachet~des Combes}.}
  \bibinfo{year}{2017}\natexlab{}.
\newblock \showarticletitle{Safe policy improvement with baseline
  bootstrapping}.
\newblock \bibinfo{journal}{\emph{arXiv preprint arXiv:1712.06924}}
  (\bibinfo{year}{2017}).
\newblock


\bibitem[\protect\citeauthoryear{Li, Chen, Kleban, and Gupta}{Li
  et~al\mbox{.}}{2015}]%
        {li2015counterfactual}
\bibfield{author}{\bibinfo{person}{Lihong Li}, \bibinfo{person}{Shunbao Chen},
  \bibinfo{person}{Jim Kleban}, {and} \bibinfo{person}{Ankur Gupta}.}
  \bibinfo{year}{2015}\natexlab{}.
\newblock \showarticletitle{Counterfactual estimation and optimization of click
  metrics in search engines: A case study}. In \bibinfo{booktitle}{\emph{WWW}}.
  ACM, \bibinfo{pages}{929--934}.
\newblock


\bibitem[\protect\citeauthoryear{Li, Chu, Langford, and Wang}{Li
  et~al\mbox{.}}{2011}]%
        {li2011unbiased}
\bibfield{author}{\bibinfo{person}{Lihong Li}, \bibinfo{person}{Wei Chu},
  \bibinfo{person}{John Langford}, {and} \bibinfo{person}{Xuanhui Wang}.}
  \bibinfo{year}{2011}\natexlab{}.
\newblock \showarticletitle{Unbiased offline evaluation of
  contextual-bandit-based news article recommendation algorithms}. In
  \bibinfo{booktitle}{\emph{WSDM}}. ACM, \bibinfo{pages}{297--306}.
\newblock


\bibitem[\protect\citeauthoryear{Liu, Swaminathan, Agarwal, and Brunskill}{Liu
  et~al\mbox{.}}{2019}]%
        {liu2019off}
\bibfield{author}{\bibinfo{person}{Yao Liu}, \bibinfo{person}{Adith
  Swaminathan}, \bibinfo{person}{Alekh Agarwal}, {and} \bibinfo{person}{Emma
  Brunskill}.} \bibinfo{year}{2019}\natexlab{}.
\newblock \showarticletitle{Off-Policy Policy Gradient with State Distribution
  Correction}.
\newblock \bibinfo{journal}{\emph{arXiv preprint arXiv:1904.08473}}
  (\bibinfo{year}{2019}).
\newblock


\bibitem[\protect\citeauthoryear{London and Sandler}{London and
  Sandler}{2019}]%
        {london2019bayesian}
\bibfield{author}{\bibinfo{person}{Ben London} {and} \bibinfo{person}{Ted
  Sandler}.} \bibinfo{year}{2019}\natexlab{}.
\newblock \showarticletitle{Bayesian Counterfactual Risk Minimization}. In
  \bibinfo{booktitle}{\emph{ICML}}. \bibinfo{pages}{4125--4133}.
\newblock


\bibitem[\protect\citeauthoryear{Strehl, Langford, Li, and Kakade}{Strehl
  et~al\mbox{.}}{2011}]%
        {strehl2010learning}
\bibfield{author}{\bibinfo{person}{Alex Strehl}, \bibinfo{person}{John
  Langford}, \bibinfo{person}{Lihong Li}, {and} \bibinfo{person}{Sham~M
  Kakade}.} \bibinfo{year}{2011}\natexlab{}.
\newblock \showarticletitle{Learning from Logged Implicit Exploration Data}. In
  \bibinfo{booktitle}{\emph{NeurIPS}}.
\newblock


\bibitem[\protect\citeauthoryear{Su, Wang, Santacatterina, and Joachims}{Su
  et~al\mbox{.}}{2019}]%
        {su2019cab}
\bibfield{author}{\bibinfo{person}{Yi Su}, \bibinfo{person}{Lequn Wang},
  \bibinfo{person}{Michele Santacatterina}, {and} \bibinfo{person}{Thorsten
  Joachims}.} \bibinfo{year}{2019}\natexlab{}.
\newblock \showarticletitle{CAB: Continuous Adaptive Blending for Policy
  Evaluation and Learning}. In \bibinfo{booktitle}{\emph{ICML}}.
  \bibinfo{pages}{6005--6014}.
\newblock


\bibitem[\protect\citeauthoryear{Sutton and Barto}{Sutton and Barto}{2018}]%
        {sutton2018reinforcement}
\bibfield{author}{\bibinfo{person}{Richard~S Sutton} {and}
  \bibinfo{person}{Andrew~G Barto}.} \bibinfo{year}{2018}\natexlab{}.
\newblock \bibinfo{booktitle}{\emph{Reinforcement learning: An introduction}}.
\newblock \bibinfo{publisher}{MIT press}.
\newblock


\bibitem[\protect\citeauthoryear{Swaminathan and Joachims}{Swaminathan and
  Joachims}{2015a}]%
        {Swaminathan/Joachims/15c}
\bibfield{author}{\bibinfo{person}{A. Swaminathan} {and} \bibinfo{person}{T.
  Joachims}.} \bibinfo{year}{2015}\natexlab{a}.
\newblock \showarticletitle{Batch Learning from Logged Bandit Feedback through
  Counterfactual Risk Minimization}.
\newblock \bibinfo{journal}{\emph{JMLR}}  \bibinfo{volume}{16}
  (\bibinfo{date}{Sep} \bibinfo{year}{2015}), \bibinfo{pages}{1731--1755}.
\newblock
\newblock
\shownote{Special Issue in Memory of Alexey Chervonenkis.}


\bibitem[\protect\citeauthoryear{Swaminathan and Joachims}{Swaminathan and
  Joachims}{2015b}]%
        {Swaminathan/Joachims/15d}
\bibfield{author}{\bibinfo{person}{A. Swaminathan} {and} \bibinfo{person}{T.
  Joachims}.} \bibinfo{year}{2015}\natexlab{b}.
\newblock \showarticletitle{The Self-Normalized Estimator for Counterfactual
  Learning}. In \bibinfo{booktitle}{\emph{NeurIPS}}.
\newblock


\bibitem[\protect\citeauthoryear{Thomas and Brunskill}{Thomas and
  Brunskill}{2016}]%
        {thomas2016data}
\bibfield{author}{\bibinfo{person}{Philip Thomas} {and} \bibinfo{person}{Emma
  Brunskill}.} \bibinfo{year}{2016}\natexlab{}.
\newblock \showarticletitle{Data-efficient Off-policy Policy Evaluation for
  Reinforcement Learning}. In \bibinfo{booktitle}{\emph{ICML}}.
\newblock


\bibitem[\protect\citeauthoryear{Wang, Agarwal, and Dudik}{Wang
  et~al\mbox{.}}{2017}]%
        {wang2016optimal}
\bibfield{author}{\bibinfo{person}{Yu-Xiang Wang}, \bibinfo{person}{Alekh
  Agarwal}, {and} \bibinfo{person}{Miroslav Dudik}.}
  \bibinfo{year}{2017}\natexlab{}.
\newblock \showarticletitle{Optimal and Adaptive Off-policy Evaluation in
  Contextual Bandits}. In \bibinfo{booktitle}{\emph{ICML}}.
\newblock


\bibitem[\protect\citeauthoryear{Watkins and Dayan}{Watkins and Dayan}{1992}]%
        {watkins1992q}
\bibfield{author}{\bibinfo{person}{Christopher~JCH Watkins} {and}
  \bibinfo{person}{Peter Dayan}.} \bibinfo{year}{1992}\natexlab{}.
\newblock \showarticletitle{Q-learning}.
\newblock \bibinfo{journal}{\emph{Machine learning}} \bibinfo{volume}{8},
  \bibinfo{number}{3-4} (\bibinfo{year}{1992}), \bibinfo{pages}{279--292}.
\newblock


\bibitem[\protect\citeauthoryear{Williams}{Williams}{1992}]%
        {williams1992simple}
\bibfield{author}{\bibinfo{person}{Ronald~J Williams}.}
  \bibinfo{year}{1992}\natexlab{}.
\newblock \showarticletitle{Simple statistical gradient-following algorithms
  for connectionist reinforcement learning}.
\newblock \bibinfo{journal}{\emph{Machine learning}} \bibinfo{volume}{8},
  \bibinfo{number}{3-4} (\bibinfo{year}{1992}), \bibinfo{pages}{229--256}.
\newblock


\end{thebibliography}
